\documentclass{article}

\usepackage[preprint]{neurips_2021}

\usepackage[utf8]{inputenc} 
\usepackage{hyperref}
\hypersetup{
    pdfnewwindow=true,     
    colorlinks=true,     
    linkcolor=blue,     
    citecolor=blue,     
    filecolor=magenta,     
    urlcolor=cyan     
}
\usepackage{url}            
\usepackage{booktabs}       
\usepackage{amsfonts}       
\usepackage{nicefrac}       
\usepackage{microtype}      
\usepackage{xcolor}         

\usepackage{graphicx}
\usepackage{subcaption}
\usepackage{amsfonts}
\usepackage{amsmath,amssymb,amsthm,thmtools}
\usepackage{bbm}
\usepackage[nolist,nohyperlinks]{acronym}
\usepackage[capitalize]{cleveref}

\usepackage[colorinlistoftodos, shadow,color=blue!30!white
					,disable
]{todonotes}
\usepackage{enumitem}

\newcommand{\R}{\mathbb{R}}

\renewcommand{\P}{\mathbf{P}}

\newcommand{\E}{\mathbf{E}}

\usepackage{algorithm,algpseudocode}

\DeclareMathOperator{\argmin}{argmin}

\newtheorem{theorem}{Theorem}

\newtheorem{lemma}{Lemma}
\newtheorem{corollary}{Corollary}

\newtheorem{proposition}{Proposition}

\newtheorem{assumption}{Assumption}

\newcommand{\todod}[2][]{\todo[size=\scriptsize,color=red!20!white,#1]{Dominic: #2}}
\newcommand{\todoi}[2][]{\todo[size=\scriptsize,color=blue!20!white,#1]{Ilja: #2}}

\begin{acronym}
\acro{RLS}{Regularised Least Squares}
\acro{ERM}{Empirical Risk Minimisation}
\acro{R-ERM}{Regularised Empirical Risk Minimisation}
\acro{RKHS}{Reproducing kernel Hilbert space}
\acro{PSD}{Positive Semi-Definite}
\acro{SGD}{Stochastic Gradient Descent}
\acro{OGD}{Online Gradient Descent}
\acro{GD}{Gradient Descent}
\acro{SGLD}{Stochastic Gradient Langevin Dynamics}
\acro{IS}{Importance Sampling}
\acro{MGF}{Moment-Generating Function}
\acro{ES}{Efron-Stein}
\acro{ESS}{Effective Sample Size}
\acro{KL}{Kullback-Liebler}
\acro{SVD}{Singular Value Decomposition}
\acro{PL}{Polyak-Łojasiewicz}
\acro{NTK}{Neural Tangent Kernel}
\acro{NTF}{Neural Tangent Feature}
\acro{KLS}{Kernelised Least-Squares}
\acro{ReLU}{Rectified Linear Unit}
\end{acronym}

\newcommand{\ba}{\mathbf{a}}
\newcommand{\bb}{\mathbf{b}}
\newcommand{\bx}{\mathbf{x}}
\newcommand{\by}{\mathbf{y}}
\newcommand{\bhy}{\mathbf{\hat{y}}}
\newcommand{\bu}{\mathbf{u}}

\newcommand{\bW}{\mathbf{W}}
\newcommand{\bw}{\mathbf{w}}
\newcommand{\bz}{\mathbf{z}}
\newcommand{\bv}{\mathbf{v}}
\newcommand{\btilW}{\widetilde{\mathbf{W}}}

\newcommand{\bK}{\mathbf{K}}
\newcommand{\bI}{\mathbf{I}}
\newcommand{\bX}{\mathbf{X}}

\newcommand{\bWstar}{\mathbf{W}^{\star}}
\newcommand{\bhW}{\mathbf{\hat{W}}}
\newcommand{\fstar}{f^{\star}}

\newcommand{\bzero}{\mathbf{0}}

\newcommand{\sB}{\mathcal{B}}

\newcommand{\sL}{\mathcal{L}}
\newcommand{\sH}{\mathcal{H}}

\newcommand{\sW}{\mathcal{W}}
\newcommand{\sX}{\mathcal{X}}
\newcommand{\sN}{\mathcal{N}}
\newcommand{\sZ}{\mathcal{Z}}
\newcommand{\sO}{\mathcal{O}}
\newcommand{\stilO}{\tilde{\mathcal{O}}}
\newcommand{\bPhi}{\mathbf{\Phi}}
\newcommand{\bphi}{\boldsymbol{\phi}}
\newcommand{\balpha}{\boldsymbol{\alpha}}

\newcommand{\Udist}{\mathrm{unif}\pr{\cbr{\pm 1/\sqrt{m}}}^m}
\newcommand{\Uset}{\cbr{\nicefrac{\pm 1}{\sqrt{m}}}^m}

\newcommand{\lmin}{\lambda_{\min}}
\newcommand{\lmax}{\lambda_{\max}}

\newcommand{\repi}{^{(i)}}

\newcommand{\deli}{^{\backslash i}}
\newcommand{\eps}{\epsilon}

\newcommand{\pr}[1]{\left( #1 \right)}
\newcommand{\br}[1]{\left[ #1 \right]}
\newcommand{\cbr}[1]{\left\{ #1 \right\}}
\newcommand{\abs}[1]{\left|#1\right|}

\newcommand{\lf}{\left}
\newcommand{\rt}{\right}
\newcommand{\bmid}{\;\middle|\;}
\newcommand{\tp}{^{\top}}
\newcommand{\ip}[1]{\left\langle #1 \right\rangle}
\newcommand{\poly}{\mathrm{poly}}
\newcommand{\ve}{\varepsilon}

\newcommand{\oracle}{\Delta_S^{\mathrm{oracle}}}
\newcommand{\egen}{\epsilon^{\mathrm{Gen}}}

\newcommand{\df}{\stackrel{\mathrm{def}}{=}}
\newcommand*\diff{\mathop{}\!\mathrm{d}}
\newcommand{\leqC}{\lesssim}
\newcommand{\geqC}{\gtrsim}
\newcommand{\diag}{\mathrm{diag}}
\newcommand{\initnoise}{\nu_{\mathrm{init}}^2}

\newcommand{\initparam}{\bW_0, \bu}

\declaretheoremstyle[
spaceabove=\topsep, spacebelow=\topsep,
headfont=\normalfont\bfseries,
notefont=\bfseries, notebraces={}{},
bodyfont=\normalfont\itshape,
postheadspace=0.5em,
name={\ignorespaces},
numbered=no,
headpunct=.]
{mystyle}
\declaretheorem[style=mystyle]{nameddef}

\title{Stability \& Generalisation of Gradient Descent for Shallow Neural Networks without the Neural Tangent Kernel}

\author{%
Dominic Richards \\
  Department of Statistics\\
  University of Oxford\\
  24-29 St Giles’, Oxford, OX1 3LB \\
  \texttt{Dominic.Richards94@gmail.com} \\  \And
  Ilja Kuzborskij \\
  DeepMind\\
  London\\
  \\
  \texttt{iljak@deepmind.com}
}

\begin{document}

\maketitle

\begin{abstract}  
  We revisit on-average algorithmic stability of \acf{GD} for training overparameterised shallow neural networks and prove new generalisation and excess risk bounds without the \ac{NTK} or \ac{PL} assumptions.
  In particular, we show oracle type bounds which reveal that the generalisation and excess risk of \ac{GD} is controlled by
  an interpolating network 
  with the shortest \ac{GD} path from initialisation (in a sense, an interpolating network with the smallest relative norm).
  While this was known for kernelised interpolants, our proof applies directly to networks trained by \ac{GD} without intermediate kernelisation.
  At the same time, by relaxing oracle inequalities developed here we recover existing \ac{NTK}-based risk bounds in a straightforward way, which demonstrates that our analysis is tighter.
  Finally, unlike most of the \ac{NTK}-based analyses we focus on regression with label noise and show
  that \ac{GD} with early stopping is \emph{consistent}.

\end{abstract}

\section{Introduction}
In a canonical statistical learning problem the learner is given a tuple of independently sampled training examples $S = \pr{z_1, \ldots, z_n}$,
where each example $z_i = (\bx_i, y_i)$ consists of an \emph{input} $\bx_i$ and \emph{label} $y_i$ jointly distributed according to some unknown probability measure $P$.
In the following we assume that inputs belong to an Euclidean ball $\sB_2^d(C_x)$ of radius $C_x$ and labels belong to $[-C_y, C_y]$.
Based on training examples the goal of the learner is to select \emph{parameters} $\bW$ from some parameter space $\sW$ in order to minimise the \emph{statistical risk}
\begin{align*}
  \sL(\bW) \df \frac12 \int \pr{f_{\bW}(\bx) - y}^2 \diff P
\end{align*}
where
$f_{\bW}$ is 
a predictor parameterised by $\bW$.
The best possible predictor in this setting is the \emph{regression function} $\fstar$, which is defined as $\fstar(\bx) = \int y \diff P_{Y|\bX=\bx}$,
while the minimum possible risk is equal to the noise-rate of the problem, which is given by 
$\sigma^2 = \int_{\sZ} (\fstar(\bx) - y)^2 \diff P$.

In this paper we will focus on a \emph{shallow neural network} predictor  
that takes the form
\begin{align*}
  f_{\bW}(\bx) \df \sum_{k=1}^m u_k \phi\pr{\ip{\bW_k, \bx}} \qquad  \pr{\bx \in \R^d, \ \bW \in \R^{d \times m}}
\end{align*}
defined with respect to some \emph{activation function} $\phi ~:~ \R \to \R$, fixed \emph{output layer} $\bu \in \Uset$, and a tunable (possibly randomised) \emph{hidden layer} $\bW$.
In particular, we will consider $f_{\bW_T}$, where the hidden layer is obtained by minimising an empirical proxy of $\sL$ called the \emph{empirical risk} $\sL_S(\bW) \df (2n)^{-1} \sum_{i=1}^n (f_{\bW}(\bx_i) - y_i)^2$ by running a \emph{\acf{GD}} procedure: For $t \in [T]$ steps with initial parameters $\bW_0$ and a step size $\eta > 0$,
we have iterates
$\bW_{t+1} = \bW_{t} - \eta \nabla \mathcal{L}_{t}(\bW_{t})$ where $\nabla \mathcal{L}_{t}(\bW)$ is the first order derivative of $\mathcal{L}_{t}(\bW)$.

Understanding the behaviour of the statistical risk for neural networks has been a long-standing topic of interest in the statistical learning theory~\citep{anthony1999neural}.
The standard approach to this problem is based on \emph{uniform} bounds on the \emph{generalisation gap}
\[
  \egen(\bW_T) \df \sL(\bW_T) - \sL_S(\bW_T)~,
\]
which, given a parameter space $\sW$, 
involves controlling the gap for the worst possible choice of $\bW \in \sW$ under some unknown data distribution. 
The theory then typically leads to the capacity-based (Rademacher complexity, VC-dimension, or metric-entropy based) bounds which hold with high probability (w.h.p.) over $S$ \citep{bartlett2002rademacher,golowich2018size,neyshabur2018role}:
\todod{added neyshabur paper here from rebuttal}
\footnote{Throughout this paper, we use $f \leqC g$ to say that there exists a universal constant $c > 0$ and some $k \in \mathbb{N}$ such that $f \leq c g \log^k(g)$ holds uniformly over all arguments.}
\begin{align}
  \label{eq:sup_gen_bound}
  \sup_{\bW \in \sW}|\egen(\bW)| \leqC \sqrt{\frac{\mathrm{capacity}(\sW)}{n}}~.
\end{align}
Thus, if one could simultaneously control the empirical risk and the capacity of the class of neural networks, one could control the statistical risk. Unfortunately, controlling the empirical risk $\mathcal{L}_{T}$ turns out to be a challenging part here since 
it is non-convex, and thus, it is not clear whether \ac{GD} can minimise it up to a desired precision.
\todod[disable]{Maybe break to new paragraph here, as we are going from "capacity bounds" to "NTK framework".}

This issue has attracted considerable attention in recent years with
numerous works \citep{du2018gradient,lee2019wide,allen2019convergence,oymak2020toward} 
demonstrating
that \emph{overparameterised} shallow networks (in a sense $m \geqC \poly(n)$) trained on subgaussian inputs converge to global minima exponentially fast, namely, $\sL_S(\bW_T) \leqC (1-\eta \cdot \nicefrac{d}{n})^T$.
Loosely speaking, these proofs are based on the idea that a sufficiently overparameterised network trained by \ac{GD} initialised at $\bW_0$ with Gaussian entries, predicts closely to a solution of a \ac{KLS} formulation minimised by \ac{GD}, where the kernel function called the \emph{\acf{NTK}}~\citep{jacot2018neural} is implicitly given by the activation function.
This connection explains the observed exponential rate in case of shallow neural networks:
For the \ac{NTK} kernel matrix $\bK$ 
the convergence rate of \ac{GD} is $(1-\eta \lmin(\bK))^T$, where $\lmin(\bK)$ is its smallest eigenvalue,
and
as it turns out,
for subgaussian inputs
$\lmin(\bK) \geqC d/n$ \citep{bartlett2021deep}.
\footnote{Which is a tightest known bound, \cite{oymak2020toward} prove a looser bound without distributional assumption on the inputs.}
Naturally, the convergence was exploited to state bounds on the statistical risk:
\cite{arora2019fine}
showed that
for 
noise-free regression ($\sigma^2 = 0$) when
$T \geqC 1/(\eta \lmin(\bK))$, 
\begin{align}
  \label{eq:arora}
  \sL(\bW_T) \leqC \sqrt{\frac{\ip{\by, (n \bK)^{-1} \by}}{n}}~.
\end{align}
Clearly, the bound is non-vacuous whenever $\ip{\by, \bK^{-1} \by} \leqC n^{2 \alpha}$ for some $\alpha \in [0,1)$, and \cite{arora2019fine} present several examples of smooth target functions which satisfy this.
More generally, for $\fstar$ which belongs to the \ac{RKHS} induced by \ac{NTK} one has $\ip{\by, \bK^{-1} \by} \leq \|\fstar\|_{\sH_{\mathrm{NTK}}}^2$ \citep{scholkopf2002learning}.
The norm-based control of the risk is standard in the literature on kernels, and thus, one might wonder to which extent neural networks are kernel methods in disguise?
\todoi[disable]{We need to compare to this}
At the same time, some experimental and theoretical evidence~\citep{bai2019beyond,seleznova2020analyzing,suzuki2021benefit} suggest that connection to kernels might be good only at explaining the behaviour of very wide networks, much more overparameterised than those used in practice.
Therefore, an interesting possibility is to develop alternative ways to analyse generalisation in neural networks: Is there a more straightforward kernel-free optimisation-based perspective?

In this paper we take a step in this direction and explore a kernel-free approach, which at the same time avoids worst-case type uniform bounds such as \cref{eq:sup_gen_bound}.
In particular, we focus on the notion of the \emph{algorithmic stability}: If an algorithm is insensitive to replacement (or removal) of an observation in a training tuple, then it must have a small generalisation gap.
Thus, a natural question is whether \ac{GD} is sufficiently stable when training overparameterised neural networks.

The stability of \ac{GD} (and its stochastic counterpart) when minimising convex and non-convex smooth objective functions was first explored by \cite{hardt2016train}.
Specifically, for a time-dependent choice of a step size $\eta_t = 1/t$ and a problem-dependent constant $\alpha \in (0,1)$ they show that
\begin{align*}
  \E\br{\egen(\bW_T) \bmid \initparam} \leqC  \ln(T) n^{-\alpha}~.
\end{align*}
Unfortunately, when combined with the \ac{NTK}-based convergence rate of the empirical risk we have a vacuous bound since $\sL_S(\bW_T) \leqC 1$.
\footnote{If $\eta_s = \frac{1}{s}$ we have $\mathcal{L}_{S}(\bW_T) \lesssim \exp(\mu \sum_{j=1}^{T} \frac{1}{j}) \approx \frac{1}{T^{\mu}}$, thus, if $\mu \approx \frac{1}{n}$ we then require $T \sim \epsilon^{-n}$ for $\mathcal{L}_{S}(\bW_T) \lesssim \epsilon$. Plugging this into the Generalisation Error bound we get that $\log(T)n^{-\alpha} = n^{1-\alpha} \log(1/\epsilon) $ which is vacuous as $n$ grows.}
This is because the stability is enforced through a quickly decaying step size rather than by exploiting a finer structure of the loss,
and turns out to be insufficient to guarantee the convergence of the empirical risk.
\todoi[disable]{Check this, the step size is time-dependent whereas we present convergence rate for a fixed step size.}
That said, several works \cite{charles2018stability,lei2021sharper} have proved stability bounds exploiting an additional structure in mildly non-convex losses. Specifically, they studied the stability of \ac{GD} minimising a \emph{gradient-dominated} empirical risk,
meaning that for all $\bW$ in some neighbourhood $\sW$ and a problem-dependent quantity $\mu$, it is assumed that $\sL_S(\bW) - \min_{\bW'}\sL_S(\bW') \leq \|\nabla \sL_S(\bW)\|^2 / (2 \mu)$.
Having iterates of \ac{GD} within $\sW$ and assuming that the gradient is $\rho$-Lipschitz, this allows to show
\begin{align*}
  \E\br{\egen(\bW_T) \bmid \initparam} \leqC  \frac{\rho}{\mu} \cdot \frac1n~.
\end{align*}
As it turns out, the condition is satisfied for the iterates of \ac{GD} training overparameterised networks~\citep{du2018gradient} with high probability over $\bW_0$. The key quantity that controls the bound is $\rho/\mu \leqC 1/\lmin(\bK)$ which can be interpreted as a \emph{condition number} of the \ac{NTK} matrix.
However, it is known that for subgaussian inputs the condition number behaves as $n/d$ which renders the bound vacuous~\citep{bartlett2021deep}.

\section{Our Contributions}
\label{sec:contributions}
In this paper we revisit algorithmic stability of \ac{GD} for training overparameterised shallow neural networks, and prove new risk bounds \textbf{without the \acf{NTK} or \acf{PL}} machinery.
In particular, we first show a bound on the generalisation gap and then specialise it to state risk bounds for
regression with and without label noise.
In the case of learning with noise we demonstrate that \ac{GD} with a form of early stopping is \emph{consistent}, meaning that the risk asymptotically converges to the noise rate $\sigma^2$.

Our analysis brings out a key quantity, which controls all our bounds, the \emph{\ac{R-ERM} Oracle} defined as
\begin{align}
  \label{eq:oracle}
  \oracle \df \min_{\bW \in \R^{d \times m}} \sL_S(\bW) + \sO\pr{\frac{\|\bW- \bW_0\|_F^2}{\eta T}} \quad \mathrm{as} \quad \eta T \to \infty~,
\end{align}
which means that $\oracle$ is essentially an empirical risk of solution closest to initialisation (an interpolant when for $m$ large enough).
We first consider a bound on the generalisation gap.
\subsection{Generalisation Gap}
\label{sec:gen_gap}
For simplicity of presentation in the following we assume that the activation function $\phi$, its first and second derivatives are bounded.
Assuming parameterisation $m \geqC (\eta T)^5$ we show (\cref{cor:gen_gap_oracle}) that the expected generalisation gap is bounded as
\begin{align}
  \label{eq:intro_gen_bound}
  \E\br{\egen(\bW_T) \bmid \initparam}
  \leq
  C \cdot
  \frac{\eta T}{n} \pr{ 1 + \frac{\eta T}{n} }
  \E\br{\oracle
  \bmid \initparam
  }~,
\end{align}
where $C$ is a constant independent from $n, T, \eta$.

Dependence of $m$ on the total number of steps $T$ might appear strange at first, however things clear out once we pay attention to the scaling of the bound.
Setting the step size $\eta$ to be constant, $T=n^{\alpha}$, and overparameterisation $m \geqC n^{5 \alpha}$ for some free parameter $\alpha \in (0, 1]$ we have
\begin{align*}
  \E\br{\egen(\bW_T) \bmid \initparam} = \sO\pr{  
  \frac{1}{n} \, \E\br{\|\bhW - \bW_0\|_F^2\bmid \initparam}
  }
  \quad \mathrm{as} \quad n \to \infty~,
\end{align*}
where $\bhW$ is chosen as parameters of a \emph{minimal-norm interpolating network}, in a sense $\bhW \in \argmin_{\bW \in \R^{d \times m}}\{\|\bW- \bW_0\|_F^2 ~|~ \sL_S(\bW) = 0 \}$.
Thus, as \cref{eq:intro_gen_bound} suggests, the generalisation gap is controlled by a minimal relative norm of an interpolating network.
In comparison, previous work in the \ac{NTK} setting obtained results of a similar type where in place of $\bhW$ one has an interpolating minimal-norm solution to an \ac{NTK} least-squares problem.
Specifically, in \cref{eq:arora} the generalisation gap is controlled by $\ip{\by, (n \bK)^{-1} \by}$, which is a squared norm of such solution.

The generalisation gap of course only tells us a partial story, since the behaviour of the empirical risk is unknown.
Next, we take care of this and present bounds \emph{excess risk} bounds.
\subsection{Risk Bound without Label Noise}
\label{sec:intro_risk_bound}
We first present a bound on the statistical risk which does not use the usual \ac{NTK} arguments to control the empirical risk.
For now assume that $\sigma^2 = 0$ meaning that there is no label noise and randomness is only in the inputs.

\paragraph{\ac{NTK}-free Risk Bound.}

In \cref{cor:risk_noise_free} we show that the risk is bounded as
\begin{align}
  \label{eq:risk_bound}
  \E\br{\sL(\bW_T) \bmid \initparam}
  \leq
  \pr{1
    +
  C \cdot \frac{\eta T}{n} \pr{ 1 + \frac{\eta T}{n} }
    }
  \E\br{\oracle
  \bmid \initparam
  }~.
\end{align}
Note that this bound looks very similar compared to the bound on generalisation gap.
The difference lies in a ``$1+$'' term which accounts for the fact that $\sL(\bW_T) \leq \oracle$ as we show in \cref{lem:OptError}.

As before we let the step size be constant, set $T=n^{\alpha}$, and let overparameterisation be $m \geqC n^{5 \alpha}$ for some $\alpha \in (0, 1]$.
Then our risk bound implies that
\todoi[disable]{It's somewhat limiting that we have to ``stop early'' even in the noiseless case. Intuitively we should be able to take $T\to \infty$ even for a fixed $m$, but the stability bound has parameterisation $m \geqC (\eta T)^5$.}
\begin{align*}
  \E\br{\sL(\bW_T) \bmid \initparam} = \sO\pr{ \frac{1}{n^{\alpha}} \, \E\br{\|\bhW - \bW_0\|_F^2\bmid \initparam}}  \quad \mathrm{as} \quad n \to \infty~,
\end{align*}
which comes as a simple consequence of bounding $\oracle$ while choosing $\bhW$ to be parameters of a minimal relative norm interpolating network as before.
Recall that $y_i = \fstar(\bx_i)$: The bound suggests that target functions $\fstar$ are only learnable if $n^{-\alpha} \, \E[\|\bhW - \bW_0\|_F^2\mid \initparam] \to 0$ as $n \to \infty$ for a fixed $\alpha$ and input distribution.
\todoi[disable]{Presenting some teacher network (or $\fstar$) would give us here something stronger...}
A natural question is whether a class of such functions is non-empty.
Indeed, the \ac{NTK} theory suggests that it is not and in the following we will recover \ac{NTK}-based results by relaxing our oracle inequality.

Similarly as in \cref{sec:gen_gap} we choose $T = n^{\alpha}$ and so we require $m \geqC n^{5 \alpha}$, which
trades off overparameterisation and convergence rate through the choice of $\alpha \in (0, 1]$.
To the best of our knowledge this is the first result of this type, where one can achieve a smaller overparameterisation compared to the \ac{NTK} literature at the expense of a slower convergence rate of the risk.
Finally, unlike the \ac{NTK} setting we do not require randomisation of the initialisation and the risk bound we obtain holds for any $\bW_0$. 
\footnote{In the \ac{NTK} literature randomisation of $\bW_0$ is required to guarantee that $\lmin(\bK) > 0$ \citep{du2018gradient}.}

\paragraph{Comparison to \ac{NTK}-based Risk Bound.}
Our generalisation bound can be naturally relaxed to obtain \ac{NTK}-based empirical risk convergence rates.
In this case we observe that our analysis is general enough to recover results of \cite{arora2019fine} (up to a difference that our bounds hold in expectation rather than in high probability).

By looking at \cref{eq:risk_bound}, our task is in controlling $\oracle$: by its definition in \cref{eq:oracle} we can see that it can be 
bounded by the regularised empirical risk of any model, and we choose a Moore-Penrose pseudo-inverse solution to the least-squares problem supplied with the \ac{NTK} feature map $\bx \mapsto (\nabla_{\bW} f_{\bW}(\bx))(\bW_0)$.
Here, entries of $\bW_0$ are drawn from $\sN(0, \initnoise)$ for some appropriately chosen $\initnoise$, and entries of the outer layer $\bu$ are distributed according to the uniform distribution over $\cbr{\pm\nicefrac{1}{\sqrt{m}}}$, independently from each other and the data.
It is not hard to see that aforementioned pseudo-inverse solution $\bW^{\mathrm{pinv}}$ will have zero empirical risk, and so we are left with controlling $\|\bW^{\mathrm{pinv}} - \bW_0\|_F^2$ as can be seen from the definition of $\oracle$.
Note that it is straightforward to do since $\bW^{\mathrm{pinv}}$ has an analytic form.
In \cref{lem:oracle_NTK} we carry out these steps and show that
with high probability over initialisation $(\initparam)$,
\begin{align*}
  \oracle = \stilO_P\pr{ \frac1n \ip{\by, (n \bK)^{-1} \by} } \quad \mathrm{as} \quad n \to \infty~,
\end{align*}
which combined with \cref{eq:risk_bound} recovers the result of \cite{arora2019fine}.
Note that we had to \emph{relax} the oracle inequality, which suggests that
the risk bound we give here is \emph{tighter}
compared to the \ac{NTK}-based bound.
Similarly, \cite{arora2019fine} demonstrated that
$
  \|\bW_T - \bW_0\|_F^2 = \stilO_P\pr{\ip{\by, (n\bK)^{-1} \by}}
$
holds w.h.p.\ over initialisation and ReLU activation function,
however their proof requires a much more involved ``coupling'' argument where iterates $(\bW_t)_{t=1}^T$ are shown to be close to the \ac{GD} iterates minimising a \ac{KLS} problem with \ac{NTK} matrix.

\subsection{Risk Bound with Label Noise and Consistency}

So far we considered regression with random inputs, but without label noise.
In this section we use our bounds to show that \ac{GD} with early stopping is \emph{consistent} in the regression with label noise.
Here labels are generated as $y_i = \fstar(\bx_i) + \ve_i$ where zero-mean random variables $(\ve_i)_{i=1}^n$ are i.i.d., almost surely \emph{bounded},\footnote{We assume boundedness of labels throughout our analysis to ensure that $\sL_S(\bW_0)$ is bounded almost surely for simplicity. This can be relaxed by introducing randomised initialisation, e.g.\ as in \citep{du2018gradient}.}
and $\E[\ve_1^2] = \sigma^2$.

For now we will use the same settings of parameters as in \cref{sec:intro_risk_bound}: Recall that $\eta$ is constant, $T=n^{\alpha}$, $m \geqC n^{5 \alpha}$ for a free parameter $\alpha \in (0,1]$.
In addition assume a \emph{well-specified} scenario which means that $m$ is large enough such that for some subset of parameters $\sW$ we achieve $\sL(\bWstar) = \sigma^2$ for some parameters $\bWstar$.
Employing our risk bound of \cref{eq:risk_bound}, we relax \ac{R-ERM} oracle as
\[
\E[\oracle ~|~ \initparam]
  \leqC
  \sigma^2 + n^{-\alpha} \|\bWstar - \bW_0\|_F^2~, \quad
  \mathrm{where} \quad \bWstar \in \argmin_{\bW \in \R^{d \times m}} \sL(\bW)~.
\]
Then we immediately have an $\alpha$-dependent risk bound
  \begin{align}
    \label{eq:risk_noise_alpha}
    \E\br{\sL(\bW_T) \bmid \initparam}
    \leq
    \pr{1
    +
    \frac{C}{n^{1-\alpha}}
    }
    \pr{\sigma^2 + \sO\pr{\frac{\|\bWstar - \bW_0\|_F^2}{n^{\alpha}}}}
     \quad \mathrm{as} \quad n \to \infty~.
  \end{align}
  It is important to note that for any tuning $\alpha \in (0,1)$,
  as long as $\bWstar$ is constant in $n$ (which is the case in the parametric setting),  
we have $\E\br{\sL(\bW_T) \bmid \initparam} \to \sigma^2$ as $n \to \infty$, and so we achieve \emph{consistency}.
Adaptivity to the noise achieved here is due to \emph{early stopping}: Indeed, recall that we take $T=n^{\alpha}$ steps and having smaller $\alpha$ mitigates the effect of the noise as can be seen from \cref{eq:risk_noise_alpha}.
This should come as no surprise as early stopping is well-known to have a regularising effect in \ac{GD} methods~\citep{yao2007early}.
The current literature on the \ac{NTK} deals with this through a kernelisation perspective by introducing an explicit $L2$ regularisation~\citep{hu2021regularization}, while risk bound of \citep{arora2019fine} designed for a noise-free setting would be vacuous in this case.

\subsection{Future Directions and Limitations}
\label{sec:future_work}
We presented generalisation and risk bounds for shallow nets controlled by the \acl{R-ERM} oracle $\oracle$.
By straightforward relaxation of $\oracle$, we showed that the risk can be controlled by
the minimal (relative) norm of an interpolating network which is tighter than the \ac{NTK}-based bounds.
There are several interesting venues which can be explored based on our results.
For example, by assuming a specific form of a target function $\fstar$ (e.g., a ``teacher'' network), one possibility is to characterise the minimal norm. This would allow us to understand better which such target function are \emph{learnable} by \ac{GD}.

One limitation of our analysis is smoothness of the activation function and its boundedness.
While boundedness can be easily dealt with by localising the smoothness analysis (using a Taylor approximation around each GD iterate) and randomising $(\initparam)$, extending our analysis to non-smooth activations (such as ReLU $\phi(x)=\max(x,0)$) appears to be non-trivial because our stability analysis crucially relies on the control of the smallest eigenvalue of the Hessian.
Another peculiarity of our analysis is a time-dependent parameterisation requirement $m \geqC (\eta T)^5$.
Making parameterisation $n$-dependent introduces an implicit early-stopping condition, which helped us to deal with label noise.
On the other hand, such requirement might be limiting when dealing with noiseless interpolating learning where one could require to have $\eta T \to \infty$ while keeping $n$ finite.
\section*{Notation}
In the following denote $\ell(\bw, (\bx, y)) \df \tfrac12 (f_{\bW}(\bx) - y)^2$.
Operator norm is denoted as $\|\cdot\|_{\text{op}}$, while $L_2$ norm is denoted by $\|\cdot\|_2$ or $ \|\cdot\|_F$.
We use notation $(a \vee b) \df \max\cbr{a, b}$ and $(a \wedge b) \df \min\cbr{a,b}$ throughout the paper.
Let $(\bW_t\repi)_t$ be the iterates of \ac{GD} obtained from the data set with a resampled data point:
\[
  S\repi \df (z_1, \ldots, z_{i-1}, \widetilde{z}_i, z_{i+1}, \ldots, z_n)
\]
where $\widetilde{z}_i$ is an independent copy of $z_i$.
Moreover, denote a remove-one version of $S$ by
\[
  S\deli \df (z_1, \ldots, z_{i-1}, z_{i+1}, \ldots, z_n)~.
\]
\section{Main Result and Proof Sketch}
\label{sec:results}
This section presents both the main results as well as a proof sketch. Precisely, Section \ref{sec:MainResult} presents the main result and Section \ref{sec:proof_sketch} the proof sketch.

\subsection{Main Results}
\label{sec:MainResult}
We formally make the following assumption regarding the regularity of the activation function. 
\begin{assumption}[Activation]
\label{ass:Activation}
The activation $\phi(u)$ is continuous and twice differentiable with constant $B_{\phi},B_{\phi^{\prime}}, B_{\phi^{\prime\prime}} \geq 0$ bounding $|\phi(u)| \leq B_{\phi}$, $|\phi^{\prime}(u)| \leq B_{\phi^{\prime}}$ and $|\phi^{\prime\prime}(u)| \leq B_{\phi^{\prime\prime}}$ for any $u \in \mathbb{R}$. 
\end{assumption}
This is satisfied for sigmoid as well as hyperbolic tangent activations. 
\begin{assumption}[Inputs, labels, and the loss function]
  \label{ass:boundedness}
  For constants $C_x, C_y, C_0 > 0$,
  inputs belong to $\sB_2^d(C_x)$, labels belong to $[-C_y, C_y]$, and loss is uniformly bounded by $C_0$ almost surely.
\end{assumption}
A consequence of the above is an essentially constant smoothness of the loss and the fact that the Hessian scales with $1/\sqrt{m}$
(see~\cref{sec:smoothness_and_curvature} for the proof).
These facts will play a key role in our stability analysis.
\footnote{We use notation $(a \vee b) \df \max\cbr{a, b}$ and $(a \wedge b) \df \min\cbr{a,b}$ throughout the paper.}
\begin{lemma}[Smoothness and curvature]
  \label{lem:eigenvalues}
  Fix $\bW, \btilW \in \R^{d \times m}$.
  Consider \cref{ass:Activation}, \cref{ass:boundedness}, and assume that
  $\sL_S(\btilW) \leq C_0^2$.
  Then, for any $S$,
\begin{align*}
  &\lmax(\nabla^2 \sL_S(\bW)) \leq \rho \quad \text{where} \quad \rho \df C_x^2 \pr{ B_{\phi^{\prime}}^2 +B_{\phi^{\prime\prime}} B_{\phi} +\frac{B_{\phi^{\prime\prime}}C_y}{\sqrt{m}}}~, \nonumber\\
  \min_{\alpha \in [0,1]}&\lambda_{\min}(\nabla^2 \sL_{S}(\btilW + \alpha(\bW - \btilW))) \geq - \frac{B_{\phi^{\prime\prime}}
    \big(B_{\phi^{\prime}}C_x  + C_0\big)}{\sqrt{m}} \cdot
    (1 \vee \|\bW - \btilW\|_F)~.
\end{align*}
\end{lemma}
Given these assumptions we are ready to present a bound on the Generalisation Error of the gradient descent iterates.

\begin{theorem}[Generalisation Error]
\label{thm:gen_gap}
Consider Assumptions \ref{ass:Activation} and \ref{ass:boundedness}.
Fix $t > 0$. If $\eta \leq 1 / (2  \rho) $ and 
 \begin{align}
 \label{equ:sufficientwidth}
   &m \geq 144 (\eta t)^2 
   C_x^4 C_0^2 B_{\phi^{\prime\prime}}^2 \pr{ 4 B_{\phi^{\prime}} C_x \sqrt{ \eta t} + \sqrt{2}}^2\\
   \mathrm{then} \qquad &\E\br{\epsilon^{\text{Gen}}(\bW_{t+1}) \bmid \initparam} \nonumber
   \leq 
   b
   \pr{ \frac{\eta }{n} + \frac{\eta^2 t}{n^2} } \sum_{j=0}^{t} \E\br{\mathcal{L}_{S}(\bW_{j}) \bmid \initparam}
 \end{align}
 where $b= 16 e^3 C_x^{\frac32}B_{\phi^{\prime}}^2(1 + C_x^{\frac32}B_{\phi^{\prime}}^2)$.
\end{theorem}
\begin{proof}
For full proof see Appendix \ref{sec:thm:gen_gap:proof} with sketch proof in Section \ref{sec:proof_sketch}. 
\end{proof}
Theorem \ref{thm:gen_gap} then provides a formal upper bound on the Generalisation Gap of a shallow neural network trained with \ac{GD}.
Specifically, provided the network is sufficiently wide the generalisation gap can be controlled through both: the gradient descent trajectory evaluated at the Empirical Risk  $\sum_{t=0}^{T} \E[\mathcal{L}_{S}(\bW_{t})~|~\initparam]$, as well as step size and number of iterations in the multiplicative factor. 
Similar bounds for the test performance of gradient descent have been given by \cite{lei2020fine}, although in our case the risk is non-convex, and thus, a more delicate bound on the Optimisation Gap is required. This is summarised within the lemma, which is presented in an ``oracle'' inequality manner. 
\begin{lemma}[Optimisation Error]
\label{lem:OptError}
Consider Assumptions \ref{ass:Activation} and \ref{ass:boundedness}.
Fix $t > 0$. If $\eta \leq 1 / (2  \rho) $, then
\begin{align*}
    \frac{1}{t} \sum_{j=0}^{t} \mathcal{L}_{S}(\bW_{j})
    \leq 
    \min_{\bW \in \mathbb{R}^{d \times m }}
    \Big\{ 
    \mathcal{L}_{S}(\bW) 
    + 
    \frac{\|\bW - \bW_{0}\|_F^2}{\eta t}
    + 
    \frac{\widetilde{b} \|\bW - \bW_0\|_F^3 }{\sqrt{m} }
    \Big\}
    + \widetilde{b}C_0 \cdot \frac{ (\eta t)^{\frac32}}{\sqrt{m}}
\end{align*}
where $\widetilde{b} = C_x^2  B_{\phi^{\prime\prime}} \pr{B_{\phi^{\prime}}C_x  + C_0}$.
\end{lemma}
\begin{proof}
  Full proof is given in Appendix \ref{sec:lem:OptError} with sketch proof in Section \ref{sec:proof_sketch}. 
\end{proof}

By combining \cref{thm:gen_gap} and \cref{lem:OptError} we get the following where $\oracle$ is defined in~\cref{eq:oracle}:
\begin{corollary}
  \label{cor:gen_gap_oracle}
  Assume the same as in \cref{thm:gen_gap} and \cref{lem:OptError}.
  Then,
  \begin{align*}
      \E\br{\egen(\bW_T) \bmid \bW_0, \bu}
  \leq
    C\cdot \frac{\eta T}{n} \pr{ 1 + \frac{\eta T}{n} }
  \E\br{\oracle
  \bmid \bW_0,\bu
  }
  \end{align*}  
  where $\oracle$ is defined in \cref{eq:oracle} and $C$ is a constant independent from $n, T, \eta$.
\end{corollary}
The above combined with \cref{lem:OptError}, and the fact that $\sL_S(\bW_T) = \min_{t \in [T]} \sL_S(\bW_{t})$ gives us:
\begin{corollary}
  \label{cor:risk_noise_free}
  Assume the same as in \cref{thm:gen_gap} and \cref{lem:OptError}.
  Then,
  \begin{align*}
    \E\br{\sL(\bW_T) \bmid \bW_0,\bu}
    \leq
    \pr{1
    +
    C \cdot \frac{\eta T}{n} \pr{ 1 + \frac{\eta T}{n} }
    }
  \E\br{\oracle
  \bmid \bW_0,\bu
  }~.
  \end{align*}
\end{corollary}
\begin{proof}
  The proof is given in \cref{sec:additional_proofs}.
\end{proof}
Finally, $\oracle$ is controlled by the norm of the \ac{NTK} solution, which establishes connection to the \ac{NTK}-based risk bounds:
\begin{theorem}[Connection between $\oracle$ and \ac{NTK}]
  \label{lem:oracle_NTK}
  Consider \cref{ass:Activation} and that $\eta T = n$.
  Moreover, assume that entries of $\bW_0$ are i.i.d., that $\lmin(\bK) \geqC 1/n$,
  and assume that $\bu \sim \Udist$ independently from all sources of randomness.
  Then, with probability least $1-\delta$ for $\delta \in (0,1)$, over $(\initparam)$,
  \begin{align*}
    \oracle
    = \stilO_P\pr{
      \frac1n \ip{\by, (n \bK)^{-1} \by}
    }
    \quad \mathrm{as} \quad n \to \infty~.
  \end{align*}
\end{theorem}
\begin{proof}
  The proof is given in \cref{sec:W_W0_NTK}.
\end{proof}

\subsection{Proof Sketch of Theorem \ref{thm:gen_gap} and Lemma \ref{lem:OptError}}
\label{sec:sketch}
\label{sec:proof_sketch}
Throughout this sketch let expectation be understood as $\E[\cdot] = \E[\cdot ~|~ \initparam]$.
For brevity we will also vectorise parameter matrices, so that $\bW \in \mathbb{R}^{dm}$ and thus $\|\cdot\|_2 = \|\cdot\|_F$.

Let us begin with the bound on the generalisation gap, for which we use the notion of algorithmic stability \citep{bousquet2002stability}.
With $\bW_t\repi$ denoting a \ac{GD} iterate with the dataset with the resampled data point $S\repi$, the Generalisation Gap can be rewritten~\citep[Chap. 13]{shalev2014understanding}
\begin{align*}
  \egen
  =
  \E[\mathcal{L}(\bW_{T})- \mathcal{L}_{S}(\bW_T)]
     = 
    \frac{1}{n} \sum_{i=1}^{n}\E[ \ell(\bW_T,\widetilde{z}_i) - \ell(\bW_T\repi,\widetilde{z}_i)]~.
\end{align*}
The equality suggests that if trained parameters do not vary much when a data point is resampled, that is $\bW_T \approx \bW_T \repi$, then the Generalisation Gap will be small.
Indeed, prior work \citep{hardt2016train,kuzborskij2018data} has been dedicated to proving uniform bounds on $L_2$ norm $\max_{i \in [n]}\|\bW_T - \bW_T \repi\|_2 $.
In our case, we consider the more refined squared $L_2$ expected stability and consider a bound for smooth losses similar to that of \cite{lei2020fine}:
\begin{lemma}
\label{lem:NNGenL2}
Consider Assumptions \ref{ass:Activation} and \ref{ass:boundedness}. Then,
\begin{align*}
  \egen
    & \leqC
    \sqrt{\E\br{\sL_{S}(\bW_{T})}}
    \sqrt{\frac{1}{n} \sum_{i=1}^{n} \E\br{ \|\bW_{T} - \bW_{T}\repi\|_{\text{op}}^2}}
    +
    \frac{1}{n} \sum_{i=1}^{n} \E\br{ \|\bW_{T}-\bW_{T}\repi\|_{\text{op}}^2}~.
\end{align*}
\end{lemma}
\begin{proof}
  The proof is given in \cref{sec:lem:NNGenL2}.
\end{proof}

To this end, we bound $\|\bW_{T}-\bW_{T}\repi\|_2^2$ recursively by using the definition of the gradient iterates and applying Young's inequality $(a+b)^2 \leq (1+\frac{1}{t})a^2 + (1+t)b^2$ to get that for any $i \in [n]$:
\begin{align}
  &\|\bW_{t+1} - \bW_{t+1} \repi\|_2^2 \label{eq:recurse_1}\\
  &=
    \|\bW_{t} - \bW_{t} \repi - \eta (\nabla \sL_{S}(\bW_t) - \nabla \sL_{S\repi}(\bW_t\repi))\|_2^2 \nonumber\\
  &=
    \|\bW_{t} - \bW_{t} \repi - \eta (\nabla \sL_{S\deli}(\bW_t) - \nabla \sL_{S\deli}(\bW_t\repi))
    + \frac{\eta}{n} (\nabla \ell(\bW_t, z_i) - \nabla \ell(\bW_t\repi, \tilde{z}_i)\|_2^2 \nonumber\\
  &\leq
  (1+\frac1t)
    \underbrace{ 
    \|\bW_{t} - \bW_{t}\repi - \eta\pr{\nabla \mathcal{L}_{S\deli}(\bW_t) -  \nabla \mathcal{L}_{S\deli}(\bW_t\repi)}\|_2^2
    }_{\text{Expansiveness of the Gradient Update}} \nonumber\\
    &+
    \frac{\eta^2(1+t)}{n^2}
    \pr{\|\nabla \ell(\bW_t,z_i)\|_2^2 + \|\nabla \ell(\bW_t\repi,\widetilde{z}_i)\|_2^2}~. \label{eq:recurse_2}
\end{align}
The analysis of the \emph{expansiveness of the gradient update} then completes the recursive relationship~\citep{hardt2016train,kuzborskij2018data,lei2020fine,richards2021weakly}.
Clearly, a trivial handling of the term (using smoothness) would yield a bound $(1+\eta \rho)^2 \|\bW_t - \bW\repi_t\|_2^2$, which leads to an exponential blow-up when unrolling the recursion.
So, we must ensure that the expansivity coefficient is no greater than one.
While the classic analysis of \cite{hardt2016train} controls this by having a polynomially decaying learning rate schedule,
here we leverage on observation of \cite{richards2021weakly} that the negative Eigenvalues of the loss's Hessian $\nabla^2 \mathcal{L}_{S^{\deli}}(\cdot) \in \mathbb{R}^{dm \times dm}$ can control the expansiveness.
This boils down to \cref{lem:eigenvalues}:
\begin{align*}
  \lmin \pr{\int_0^1 \nabla^2 \sL_{S\deli}(\bW_t\repi + \alpha (\bW_t - \bW_t\repi)) \diff \alpha}
  \geqC
    - \frac{1 \vee \|\bW_t - \bW_t\repi\|_2}{\sqrt{m}}
  \geqC
  - \sqrt{\frac{\eta t}{m}}
\end{align*}
where $\|\bW_t - \bW_t\repi\|_2 \leqC \sqrt{\eta t}$ comes from smoothness and standard analysis of \ac{GD} (see \cref{lem:Useful}).
Note that the eigenvalue scales $1/\sqrt{m}$ because Hessian has a block-diagonal structure.
\paragraph{Controlling expansiveness of \ac{GD} updates.}
Abbreviate $\Delta = \nabla \sL_{S\deli}(\bW_t) - \nabla \sL_{S\deli}(\bW_t\repi)$ and $\nabla^2 = \int_0^1 \nabla^2 \sL_{S\deli}(\bW_t\repi + \alpha (\bW_t - \bW_t\repi)) \diff \alpha$.
Now we ``open'' the squared norm
\begin{align*}
  \|\bW_{t} - \bW_{t}\repi - \eta \Delta \|_2^2
  =
    \|\bW_{t} - \bW_{t}\repi\|_2^2
  + \eta^2 \lf\|\Delta \rt\|_2^2
   - 2 \eta \ip{\bW_{t} - \bW_{t}\repi, \Delta}
\end{align*}
and use Taylor's theorem for gradients to have
\begin{align*}
  \eta^2 \lf\|\Delta \rt\|_2^2 - 2 \eta \ip{\bW_{t} - \bW_{t}\repi, \Delta}
  &=
    \ip{\Delta, \pr{ \eta^2 \nabla^2  - 2 \eta \bI} (\bW_t - \bW_t\repi)}\\
  &=
    \ip{\bW_t - \bW_t\repi, \eta \nabla^2 \pr{ \eta \nabla^2  - 2 \bI} (\bW_t - \bW_t\repi)}\\
  &\leqC
    \pr{\eta \sqrt{\frac{\eta t}{m}} + \eta^2 \frac{\eta t}{m}} \|\bW_t - \bW_t\repi\|_2^2
\end{align*}
where we noted that $(\eta \nabla^2  - 2 \bI)$ has only negative eigenvalues due to assumption $\eta \leq 1 / (2 \rho)$.

By rearranging this inequality one can verify that the gradient operator is \emph{approximately co-coercive}~\citep{hardt2016train}.
This result, crucial to our proof, is supported by some empirical evidence: In \cref{fig:monotone_gd} we show a synthetic experiment where we train a shallow neural network with sigmoid activation. 
\begin{align*}
  \|\bW_{t} - \bW_{t}\repi - \eta \Delta \|_2^2
  \leqC
  \pr{1 + \eta \sqrt{\frac{\eta t}{m}}} \|\bW_t - \bW_t\repi\|_2^2~.
\end{align*}
Plugging this back into \cref{eq:recurse_1}-\eqref{eq:recurse_2} we unroll the recursion:
\begin{align*}
  \|\bW_{t+1} - \bW_{t+1} \repi\|_2^2
     & \leqC 
       \frac{\eta^2 t}{n^2}
       \sum_{j=0}^{t}
       \pr{1 + \frac1t}^t \pr{1 + \eta \sqrt{\frac{\eta t}{m}}}^t
       \pr{ \|\nabla \ell(\bW_j,z_i)\|_2^2 + \|\nabla \ell(\bW_j\repi,\widetilde{z}_i)\|_2^2 }~.
\end{align*}
The above suggest that to prevent exponential blow-up (and ensure that the expansiveness remains under control), it's enough to have $\eta \sqrt{\frac{\eta t}{m}} \leq 1/t$, which is an assumption of the theorem, enforcing relationship between $\eta t$ and $m$.
\paragraph{On-average stability bound.}
Now we are ready to go back to \cref{lem:NNGenL2} and complete the proof of the generalisation bound.
Since we have a squared loss we have $\|\nabla \ell(\bW_j,z_i)\|_2^2 \leqC \ell(\bW_{j},z_i)$, and thus, when summing over $i \in [n]$ we recover the Empirical risk  $\frac{1}{n} \sum_{i=1}^{n} \|\nabla \ell(\bW_j,z_i)\|_2^2 \leqC \mathcal{L}_{S}(\bW_j)$. Finally, noting that $\E[ \ell(\bW_{j},z_i)] = \E[ \ell(\bW_{j}\repi,\widetilde{z}_i)] $ we then arrive at the following bound on the expected squared $L_2$ stability
\begin{align*}
    \frac{1}{n} \sum_{i=1}^{n}\E[ \|\bW_{t+1} - \bW_{t+1} \repi\|_2^2]
    \lesssim  \frac{\eta^2 t}{n^2} \sum_{j=0}^{t} \E[\mathcal{L}_{S}(\bW_j)]~.
\end{align*}
Note that the Generalisation Gap is directly controlled by the Optimisation Performance along the path of the iterates, matching the shape of bounds by \cite{lei2020fine}.
Using smoothness of the loss as well as that  $\mathcal{L}_{S}(\bW_{T}) \leq \frac{1}{T} \sum_{j=0}^{T} \mathcal{L}_{S}(\bW_{j})$, the risk is bounded when $m \gtrsim (\eta t)^{3}$ as
\begin{align*}
    \E[\mathcal{L}(\bW_{T})]
    \lesssim
    \pr{1 + \frac{\eta T}{n}\pr{1 + \frac{\eta T}{n}}} 
    \frac{1}{T} \sum_{j=0}^{T} \E[\mathcal{L}_{S}(\bW_j)]~.
\end{align*}
\paragraph{Optimisation error.}
We must now bound the Optimisation Error averaged across the iterates. Following standard optimisation arguments for gradient descent on a smooth objective we then get
\begin{align*}
    \mathcal{L}_{S}(\bW_{t+1}) 
    \leq 
    \mathcal{L}_{S}(\bW_{t}) - \frac{\eta}{2} \|\nabla \mathcal{L}_{S}(\bW_{t})\|_2^2
\end{align*}
At this point convexity would usually be applied to upper bound $\mathcal{L}_{S}(\bW_{t}) \leq \mathcal{L}_{S}(\bhW) - \langle \nabla \mathcal{L}_{S}(\bW_{t}), \bhW -  \bW_{t}\rangle $ where $\bhW$ is a minimiser of $\mathcal{L}_{S}(\cdot)$. The analysis differs from this approach in two respects. Firstly, we do not consider a minimiser of $\mathcal{L}_{S}(\cdot)$ but a general point $\bW$, which is then optimised over at the end. Secondly, the objective is not convex, and therefore, we leverage that the Hessian negative Eigenvalues are on the order of $-\frac{1}{\sqrt{m}}$ to arrive at the upper bound (\cref{lem:eigenvalues}):
\begin{align*}
    \mathcal{L}_{S}(\bW_{t})
    \lesssim 
    \underbrace{ 
    \mathcal{L}_{S}(\bW) 
    - 
    \langle \nabla \mathcal{L}_{S}(\bW_{t}), \bW - \bW_{t}\rangle}_{\text{Convex Component}} + 
    \underbrace{ 
    \frac{1}{\sqrt{m}} \|\bW - \bW_{t}\|_2^3
    }_{\text{Negative Hessian Eigenvalues}}
\end{align*}
After this point, the standard optimisation analysis for gradient descent can be performed i.e.  
$- \langle \nabla \mathcal{L}_{S}(\bW_{t}), \bW - \bW_{t}\rangle  - \frac{\eta}{2} \|\mathcal{L}_{S}(\bW_{t})\|_2^2 = \frac{1}{\eta} \pr{ \|\bW - \bW_{t} \|_2^2 - \|\bW - \bW_{t+1}\|_2^2}$, which then yields to a telescoping sum over $t=0,1,\dots,T-1$. This leads to the upper bound for $\bW \in \mathbb{R}^{m \times d}$ 
\begin{align*}
     \frac{1}{T} \sum_{j=0}^{T} \E[\mathcal{L}_{S}(\bW_j)]
     \lesssim 
     \mathcal{L}_{S}(\bW) + \frac{\|\bW - \bW_{0}\|_2^2}{\eta T} 
     + 
     \frac{1}{\sqrt{m}} \cdot \frac{1}{T} \sum_{j=0}^{T} \|\bW - \bW_{j}\|_2^3
\end{align*}
where we plug this into the generalisation bound and take a minimum over $\bW \in \mathbb{R}^{m d}$.
\begin{figure}
  \centering
  \includegraphics[width=9cm]{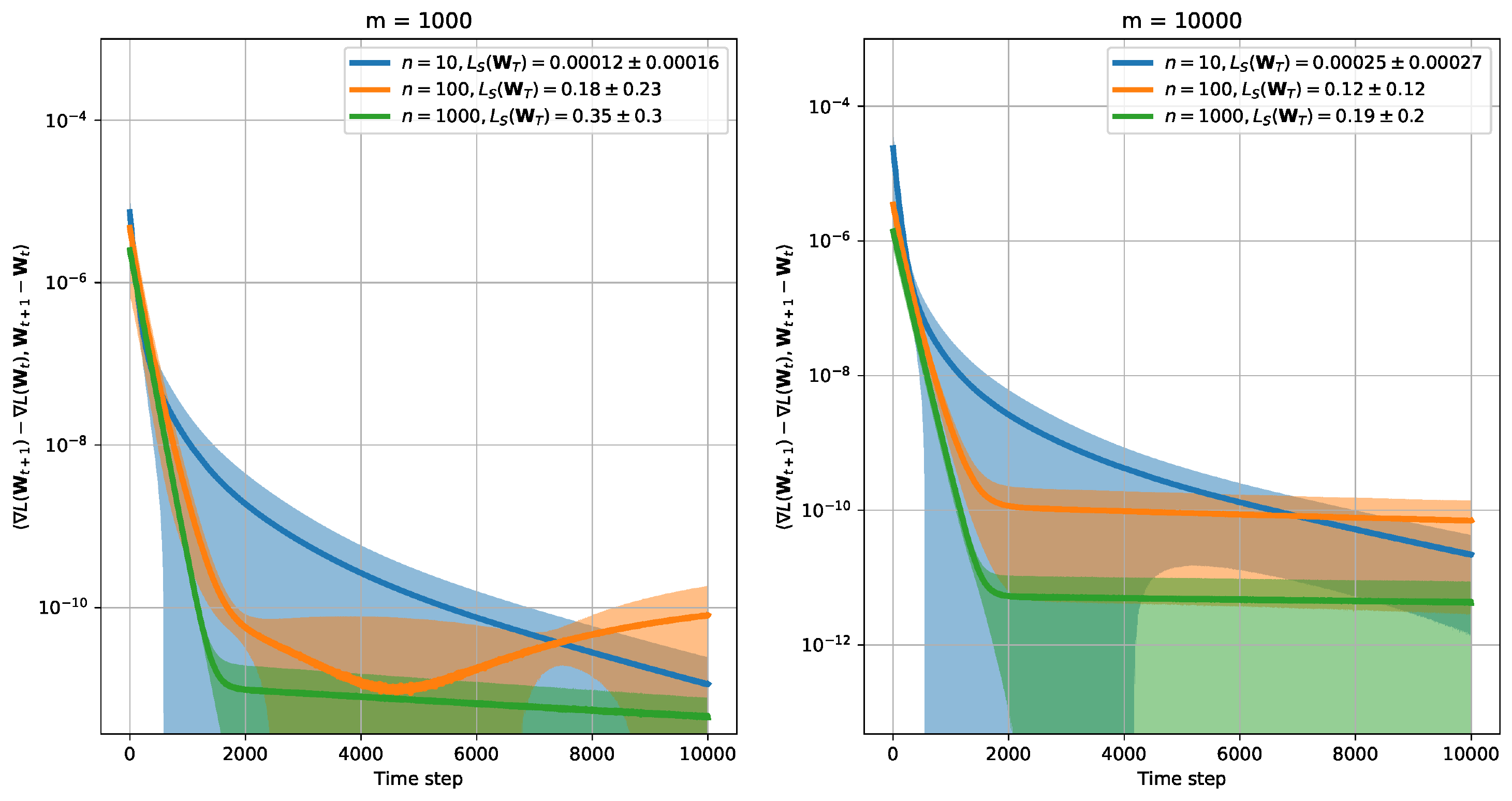}
  \caption{The gradient operator of an overparameterised shallow neural network is almost-monotone: Here, inputs are uniformly distributed on a $10$-sphere and labels are generated by $\bx \mapsto \frac{e^{\ip{\bWstar, \bx}}}{1+e^{\ip{\bWstar, \bx}}}$ where $\bWstar$ is once sampled from $\sN(\bzero, \bI_{10})$ at the beginning of the training. The initialisation $\bW_0$ is sampled from $\sN(\bzero, \bI_{10})$ and each experiment is performed $10$ times (which is reflected by standard deviation).
    The shallow network is then trained by \ac{GD} with $\eta = 1$, $T=10^4$, and sigmoid activation.
  }
  \label{fig:monotone_gd}
\end{figure}

\section{Additional Related Literature}

\textbf{Stability of Gradient Descent}. Algorithmic stability \citep{bousquet2002stability} has been used to investigate the generalisation performance of \ac{GD} in a number of works \citep{hardt2016train,kuzborskij2018data,chen2018stability,lei2020fine,richards2021weakly}.
While near optimal bounds are achieved for convex losses, the non-convex case aligning with this work is more challenging. Within \citep{hardt2016train,kuzborskij2018data} in particular a restrictive $1/t$ step size is required if $t$ iterations of gradient descent are performed.
This was partially alleviated within \citep{richards2021weakly} which demonstrated that the magnitude of Hessian's negative eigenvalues can be leveraged to allow for much larger step sizes to be taken.
In this work we also utilise the magnitude of Hessian's negative Eigenvalue, although a more delicate analysis is required as both: the gradient is potentially unbounded in our case; and the Hessian's negative eigenvalues need to be bounded when evaluated across multiple points.
We note the scaling of the Hessian network has also previously been observed within \citep{liu2020linearity,liu2020loss}, with empirical investigations into the magnitude of the Hessian's negative Eigenvalues conducted within \citep{sagun2016eigenvalues,yuan2018stagewise}.
Stability-based generalisation bounds were also shown for certain types of non-convex functions, such \acl{PL} functions~\citep{charles2018stability,lei2021sharper}
and functions with strict saddles among stationary points \citep{gonen2017fast}.
In line with our results, \cite{rangamani2020interpolating} shows that for interpolating kernel predictors, minimizing the norm of the \acs{ERM} solution minimizes stability.

Several works have also recently demonstrated consistency of \ac{GD} with early stopping for training shallow neural networks in the presence of label noise.
The concurrent work \citep{ji2021early} also showed that shallow neural networks trained with gradient descent are consistent, although their approach distinctly different to ours e.g. leveraging structure of logistic loss as well as connections between shallow neural networks and random feature models / NTK.
Earlier, \cite{li2020gradient} showed consistency under a certain gaussian-mixture type parametric classification model.
In a more general nonparametric setting, \cite{kuzborskij2021nonparametric} showed that early-stopped \ac{GD} is consistent when learning Lipschitz regression functions.

Certain notions of stability, such the uniform stability (taking $\sup$ over the data rather than expectation) allows to prove \textbf{high-probability risk bounds (h.p.)}~\citep{bousquet2002stability} which are known to be optimal up to log-terms~\citep{feldman2019high,bousquet2020sharper}.
In this work we show bounds in expectation, or equivalently, control the first moment of a generalisation gap.
To prove a h.p. bounds while enjoying benefits of a stability in expectation, one would need to control higher order moments~\citep{maurer2017second,abou2019exponential}, however often this is done through a higher-order uniform stability: Since our proof closely relies on the stability in expectation, it is not clear whether our results can be trivially stated with high probability.

\section*{Acknowledgements}
D.R. is supported by the EPSRC and MRC through the
OxWaSP CDT programme (EP/L016710/1), and the London Mathematical Society ECF-1920-61.

\bibliographystyle{plainnat}
\bibliography{References}

\appendix

\newpage

\section*{Notation}
In the following denote $\ell(\bw, (\bx, y)) \df \tfrac12 (f_{\bW}(\bx) - y)^2$.
Unless stated otherwise, we work with vectorised quantities so $\bW \in \mathbb{R}^{dm}$ and therefore simply interchange $\|\cdot\|_2$ with $ \|\cdot\|_F$.
We also use notation $(\bW)_k$ so select $k$-th block of size $d$, that is $(\bW)_k = [W_{(d-1)k+1}, \ldots, W_{dk}]\tp$.
We use notation $(a \vee b) \df \max\cbr{a, b}$ and $(a \wedge b) \df \min\cbr{a,b}$ throughout the paper.
Let $(\bW_t\repi)_t$ be the iterates of \ac{GD} obtained from the data set with a resampled data point:
\[
  S\repi \df (z_1, \ldots, z_{i-1}, \widetilde{z}_i, z_{i+1}, \ldots, z_n)
\]
where $\widetilde{z}_i$ is an independent copy of $z_i$.
Moreover, denote a remove-one version of $S$ by
\[
  S\deli \df (z_1, \ldots, z_{i-1}, z_{i+1}, \ldots, z_n)~.
\]

\section{Smoothness and Curvature of the Empirical Risk (Proof of \cref{lem:eigenvalues})}
\label{sec:smoothness_and_curvature}

\begin{nameddef}[\cref{lem:eigenvalues} (restated)]
  Fix $\bW, \btilW \in \R^{d \times m}$.
  Consider \cref{ass:Activation}, \cref{ass:boundedness}, and assume that
  $\sL_S(\btilW) \leq C_0^2$.
  Then, for any $S$,
\begin{align}
  &\lmax(\nabla^2 \sL_S(\bW)) \leq \rho \quad \text{where} \quad \rho \df C_x^2 \pr{ B_{\phi^{\prime}}^2 +B_{\phi^{\prime\prime}} B_{\phi} +\frac{B_{\phi^{\prime\prime}}C_y}{\sqrt{m}}}~, \nonumber\\
  \min_{\alpha \in [0,1]}&\lambda_{\min}(\nabla^2 \sL_{S}(\btilW + \alpha(\bW - \btilW))) \geq - \frac{B_{\phi^{\prime\prime}}
    \big(B_{\phi^{\prime}}C_x  + C_0\big)}{\sqrt{m}} \cdot
                           (1 \vee \|\bW - \btilW\|_F)~. \label{equ:EigBound}
\end{align}
\end{nameddef}

\begin{proof}
Vectorising allows the loss's Hessian to be denoted
\begin{align}
\label{equ:LossHessian}
    \nabla^2 \ell(\bW,z) 
    = 
    \nabla f_{\bW}(\bx) \nabla f_{\bW}(\bx)^{\top} 
    + 
    \nabla^2 f_{\bW}(\bx) (f_{\bW}(\bx) - y)
\end{align}
where 
\begin{align*}
    \nabla f_{\bW}(\bx) 
    = 
    \begin{pmatrix}
    u_1 \bx \phi^{\prime}\pr{\langle (\bW)_1,\bx\rangle}  \\
    u_2 \bx \phi^{\prime}\pr{\langle (\bW)_2,\bx\rangle} \\
    \vdots \\
    u_m \bx \phi^{\prime}\pr{\langle (\bW)_m,\bx\rangle}
    \end{pmatrix}
    \in \mathbb{R}^{dm}
\end{align*}
and $\nabla^2 f_{\bW}(\bx)  \in \mathbb{R}^{dm \times dm}$ with
\begin{align*}
    \nabla^2 f_{\bW}(\bx)
    = 
    \begin{pmatrix}
    u_1 \bx \bx^{\top} \phi^{\prime\prime}(\langle (\bW)_1,\bx \rangle )  & 0 & 0 & \dots & 0 \\
    0 & u_2 \bx \bx^{\top} \phi^{\prime\prime}(\langle (\bW)_2,\bx \rangle )& 0 &  \dots & 0 \\
    \vdots & \ddots & \ddots & \vdots & \vdots \\
    0 & 0 & 0 & \dots  & u_m \bx \bx^{\top} \phi^{\prime\prime}(\langle (\bW)_m,\bx \rangle ) 
    \end{pmatrix}
\end{align*}
Note that we then immediately have with $\bv = (\bv_1,\bv_2,\dots,\bv_m) \in \mathbb{R}^{dm}$ with $\bv_i \in \mathbb{R}^{d}$ 
\begin{align}
\label{equ:HessianBound}
    \|\nabla^2 f_{\bW}(\bx)\|_2 
    & = \max_{\bv: \|\bv\|_2 \leq 1} 
    \sum_{j=1}^{m} u_j \langle \bv_j,\bx\rangle^2 \phi^{\prime\prime}(\langle (\bW)_j,\bx \rangle)
    \nonumber \\
    & \leq 
    \frac{1}{\sqrt{m}} 
    \|\bx\|_2^2 
    B_{\phi^{\prime\prime}}
    \max_{\bv: \|\bv\|_2 \leq 1} 
    \sum_{j=1}^{m} \|\bv_j\|_2^2 \nonumber 
    \\
    & \leq  
    \frac{C^2_x  B_{\phi^{\prime\prime}}}{\sqrt{m}}~.
\end{align}
We then see that the maximum Eigenvalue of the Hessian is upper bounded for any $\bW \in \mathbb{R}^{dm}$, that is
\begin{align}
  \label{eq:grad_is_lip_smooth}
    \|\nabla^2 \ell(\bW,z) \|_2
    & \leq 
    \|\nabla f_{\bW}(\bx)\|_2^2 
    + 
    \|\nabla^2 f_{\bW}(\bx)\|_2|f_{\bW}(\bx) - y|\\
    & \leq 
    C_{x}^2 B_{\phi^{\prime}}^2 + \frac{C_x^2 B_{\phi^{\prime\prime}}}{\sqrt{m}}( \sqrt{m} B_{\phi} + C_y)
\end{align}
and therefore the objective is $\rho$-smooth with $\rho = C_x^2  \big( B_{\phi^{\prime}}^2 + B_{\phi^{\prime\prime}} B_{\phi} +\frac{B_{\phi^{\prime\prime}} C_y}{\sqrt{m}}\big)$.

Let us now prove the lower bound \eqref{equ:EigBound}.
For some fixed $\bW, \btilW \in \R^{d \times m}$ define
\begin{align*}
  \bW(\alpha) \df \btilW + \alpha(\bW - \btilW) \qquad \alpha \in [0,1]~.
\end{align*}
Looking at the Hessian in \eqref{equ:LossHessian}, the first matrix is positive semi-definite, therefore
\begin{align*}
    \lambda_{\min}(\nabla^2 \sL_{S}(\bW(\alpha))) 
    & \geq 
    - \Big( \max_{i=1,\dots,n} \big\{ \|\nabla^2 f_{\bW(\alpha)}(\bx_i)\|_2\big\} \Big)
    \frac{1}{n}\sum_{i=1}^{n}|f_{\bW(\alpha)}(\bx_i) - y_i|\\
    & \geq - \frac{C_x^2 B_{\phi^{\prime\prime}}}{\sqrt{m}}
      \cdot
    \frac{1}{n}\sum_{i=1}^{n}|f_{\bW(\alpha)}(\bx_i) - y_i|
\end{align*}
where we have used the upper bound on $\|\nabla^2 f_{\bW}(\bx_i)\|_2 $. Adding and subtracting $f_{\btilW}(\bx_i)$ inside the absolute value we then get 
\begin{align*}
    \frac{1}{n}\sum_{i=1}^{n}|f_{\bW(\alpha)}(\bx_i) - y_i|
    &\leq 
    \frac{1}{n}\sum_{i=1}^{n}|f_{\bW(\alpha)}(\bx_i) - f_{\btilW}(\bx_i)| + \frac{1}{n}\sum_{i=1}^{n} |f_{\btilW}(\bx_i) - y_i|\\
    & \leq 
    B_{\phi^{\prime}}C_x \|\bW(\alpha) - \btilW\|_2
    + 
    \sqrt{ \sL_{S}(\btilW)} \\
    & \leq 
    B_{\phi^{\prime}}C_x \|\bW(\alpha) - \btilW\|_2
    + 
    \sqrt{\sL_{S}(\bW_0)} \\
    & \leq 
     \big(B_{\phi^{\prime}}C_x  + C_0\big)(1 \vee \|\bW(\alpha) - \btilW\|_2)
\end{align*}
where for the second term we have simply applied Cauchy-Schwarz inequality.
For the first term, we used that for any $\bW,\widetilde{\bW} \in \mathbb{R}^{d m}$ we see that
\begin{align}
  \label{eq:f_diff_1}
    |f_{\bW}(\bx) - f_{\widetilde{\bW}}(\bx)| 
    & \leq 
    \frac{1}{\sqrt{m}} \sum_{i=1}^{m} |\phi(\langle (\bW)_i,\bx\rangle) - \phi(\langle (\widetilde{\bW})_i,\bx\rangle)|\\
    & \leq
    \frac{B_{\phi^{\prime}}}{\sqrt{m}}\sum_{i=1}^{m}
    \abs{\langle (\bW)_i - (\widetilde{\bW})_i,\bx\rangle} \nonumber \\
    & \leq 
    C_x B_{\phi^{\prime}}\|\bW - \widetilde{\bW}\|_2. \label{eq:f_diff_2}
\end{align}
Bringing everything together yields the desired lower bound 
\begin{align*}
    \lambda_{\min}(\nabla^2 \sL_{S}(\bW(\alpha)))
  &\geq 
    -
    \frac{C_x^2}{\sqrt{m}} B_{\phi^{\prime\prime}}
  \big(B_{\phi^{\prime}}C_x  + C_0\big)(1 \vee \|\bW(\alpha) - \btilW\|_2)\\
  &\geq 
    -
    \frac{C_x^2}{\sqrt{m}} B_{\phi^{\prime\prime}}
    \big(B_{\phi^{\prime}}C_x  + C_0\big)(1 \vee \|\bW - \btilW\|_2)~.
\end{align*}
This holds for any $\alpha \in [0,1]$, therefore, we took the minimum. 
\end{proof}

\section{Optimisation Error Bound (Proof of Lemma \ref{lem:OptError}) }
\label{sec:lem:OptError}
In this section we present the proof for the Optimisation Error term. We begin by quoting the result which we set to prove. 
\begin{nameddef}[\cref{lem:OptError} (restated)]
Consider Assumptions \ref{ass:Activation} and \ref{ass:boundedness}.
Fix $t > 0$. If $\eta \leq 1 / (2  \rho) $, then
\begin{align*}
    \frac{1}{t} \sum_{j=0}^{t} \mathcal{L}_{S}(\bW_{j})
    \leq 
    \min_{\bW \in \mathbb{R}^{d \times m }}
    \Big\{ 
    \mathcal{L}_{S}(\bW) 
    + 
    \frac{\|\bW - \bW_{0}\|_F^2}{\eta t}
    + 
    \frac{\widetilde{b} \|\bW - \bW_0\|_F^3 }{\sqrt{m} }
    \Big\}
    + \widetilde{b}C_0 \cdot \frac{ (\eta t)^{\frac32}}{\sqrt{m}}
\end{align*}
where $\widetilde{b} = C_x^2  B_{\phi^{\prime\prime}} \pr{B_{\phi^{\prime}}C_x  + C_0}$.
\end{nameddef}
\begin{proof}
Using \cref{lem:eigenvalues} as well as that $\eta \rho \leq 1$ from the assumption within the theorem yields for $t \geq 0 $
\begin{align*}
    \sL_{S}(\bW_{t+1}) 
    & \leq 
    \sL_{S}(\bW_t)
    - \eta\big(1 - \frac{\eta \rho}{2}\big)\|\nabla \sL_{S}(\bW_t)\|_2^2\\
    & \leq 
    \sL_{S}(\bW_t)
    - \frac{\eta}{2} \|\nabla \sL_{S}(\bW_t)\|_2^2.
\end{align*}
Fix some $\bW \in \mathbb{R}^{d m}$. We then use the following inequality which will be proven shortly:
\begin{align}
\label{equ:Inequality}
    \sL_{S}(\bW_t) 
    \leq 
    \sL_{S}(\bW) 
    - 
    \langle \bW - \bW_t, \nabla \sL_{S}(\bW_t)\rangle 
    + 
    \frac{\widetilde{b}}{\sqrt{m}} 
    \big( 1 \vee \|\bW - \bW_t\|_2\big)^3
\end{align}
Plugging in this inequality we then get 
\begin{align*}
    \sL_{S}(\bW_{t+1}) 
    \leq
    \sL_{S}(\bW) 
    - 
    \langle \bW - \bW_t, \nabla \sL_{S}(\bW_t)\rangle
    - \frac{\eta}{2}\|\nabla \sL_{S}(\bW_t)\|_2^2 
    + 
    \frac{\widetilde{b}}{\sqrt{m}} 
    \big( 1 \vee \|\bW - \bW_t\|_2\big)^3~.
\end{align*}
Note that we can rewrite 
\begin{align*}
    & -  \langle \bW - \bW_t, \nabla \sL_{S}(\bW_t)\rangle
    - \frac{\eta}{2}\|\nabla \sL_{S}(\bW_t)\|_2^2 \\
    & = 
    \frac{1}{\eta} \langle \bW - \bW_t,\bW_{t+1} - \bW_t\rangle 
    - 
    \frac{1}{2\eta}\|\bW_{t+1} - \bW_t\|_2^2\\
    & = 
    \frac{1}{\eta} \big( \|\bW - \bW_t\|_2^2 -\|\bW_{t+1}-\bW\|_2^2\big)
\end{align*}
where we used that for any vectors $\bx, \by, \bz$: $2 \langle \bx-\by,\bx-\bz\rangle = \|\bx-\by\|_2^2 + \|\bx-\bz\|_2^2 - \|\by-\bz\|_2^2$ (which is easier to see if we relabel $2\langle \ba,\bb\rangle = \|\ba\|_2^2 + \|\bb\|_2^2 - \|\ba -\bb \|_2^2$). Plugging in and summing up we get 
\begin{align*}
    \frac{1}{t}\sum_{s=0}^{t}
    \sL_{S}(\bW_t) 
    \leq 
    \sL_{S}(\bW) 
    +
    \frac{\|\bW - \bW_0\|_2^2}{\eta t}
    + 
  \frac{\widetilde{b}}{\sqrt{m}}
  \cdot
    \frac{1}{t}\sum_{s=0}^{t}
    \big( 1 \vee \|\bW - \bW_t\|_2\big)^3.
\end{align*}
Since the choice of $\bW$ was arbitrary, we can simply take the minimum. 

\paragraph{Proof of \cref{equ:Inequality}.} Let us now prove the key \cref{equ:Inequality}. Fix $t \geq 0 $, and let us define the following functions for $\alpha \in [0,1]$
\begin{align*}
  \bW(\alpha) & \df \bW_t + \alpha(\bW - \bW_t)~,\\
    g(\alpha) & \df \sL_{S}(\bW(\alpha)) 
    + 
    \frac{\widetilde{b}}{\sqrt{m}} \cdot \frac{\alpha^2}{2} \big(1 \vee \|\bW - \bW_t\|_2\big)^3~.
\end{align*}
Note that computing the derivative we have  
\begin{align*}
    g^{\prime\prime}(\alpha) 
    = 
    (\bW - \bW_t)^{\top}\nabla^2 \sL_{S}(\bW(\alpha))  (\bW - \bW_t)
    + 
    \frac{\widetilde{b} }{\sqrt{m}} \big(1 \vee \|\bW - \bW_t\|_2\big)^3.
\end{align*}
On the other hand by \cref{lem:eigenvalues} we have
\begin{align*}
\min_{\alpha \in [0,1]}\lambda_{\min}(\nabla^2 \sL_{S}(\bW(\alpha))) \geq - \frac{\widetilde{b} }{\sqrt{m}} \big(1 \vee \|\bW - \bW_t\|_2\big)
\end{align*}
and we immediately have $g^{\prime\prime}(\alpha)  \geq 0$, and thus, $g(\cdot)$ is convex on $[0,1]$. Inequality \eqref{equ:Inequality} then arises from $g(1) - g(0) \geq g^{\prime}(0)$, in particular 
\begin{align*}
    g(1) - g(0) 
    & = 
    \sL_{S}(\bW) + \frac{\widetilde{b} }{\sqrt{m}} \big(1 \vee \|\bW - \bW_t\|_2\big)^3  - \sL_S(\bW_t) \\
    & \geq 
    \langle \bW - \bW_t,\nabla \sL_S(\bW_t)\rangle\\
    & = g^{\prime}(0)
\end{align*}
as required. 
\end{proof}

\section{Generalisation Gap Bound (Proof of Theorem \ref{thm:gen_gap}) }
\label{sec:thm:gen_gap:proof}
In this section we prove:
\begin{nameddef}[\cref{thm:gen_gap} (restated)]
Consider Assumptions \ref{ass:Activation} and \ref{ass:boundedness}.
Fix $t > 0$. If $\eta \leq 1 / (2  \rho) $ and 
 \begin{align*}
   &m \geq 144 (\eta t)^2 
   C_x^4 C_0^2 B_{\phi^{\prime\prime}}^2 \pr{ 4 B_{\phi^{\prime}} C_x \sqrt{ \eta t} + \sqrt{2}}^2\\
   \mathrm{then} \qquad &\E\br{\epsilon^{\text{Gen}}(\bW_{t+1}) \bmid \initparam} \nonumber
   \leq 
   b
   \pr{ \frac{\eta }{n} + \frac{\eta^2 t}{n^2} } \sum_{j=0}^{t} \E\br{\mathcal{L}_{S}(\bW_{j}) \bmid \initparam}
 \end{align*}
 where $b= 16 e^3 C_x^{\frac32}B_{\phi^{\prime}}^2(1 + C_x^{\frac32}B_{\phi^{\prime}}^2)$.
\end{nameddef}

To prove this result we use algorithmic stability arguments.
Recall that we can write \citep[Chapter 13]{shalev2014understanding},
\begin{align*}
    \E\br{\sL(\bW_{t+1}) - \sL_{S}(\bW_{t+1}) \bmid \initparam}
    = 
    \frac{1}{n} \sum_{i=1}^{n} \E\br{ \ell(\bW_{t+1},\widetilde{z}_i) - \ell(\bW_{t+1}\repi,\widetilde{z}_i) \bmid \initparam}~.
\end{align*}
The following lemma shown in \cref{sec:lem:NNGenL2} then bounds the Generalisation error in terms of a notation of stability. 
\begin{nameddef}[\cref{lem:NNGenL2} (restated)]
Consider Assumptions \ref{ass:Activation} and \ref{ass:boundedness}. Then, for any $t \geq 0$,
\begin{align*}
  &\E\br{\sL(\bW_{t}) - \sL_{S}(\bW_{t})\bmid \initparam}\\
    & \leq 
    B_{\phi^{\prime}} \sqrt{ C_x} \sqrt{\E\br{\sL_{S}(\bW_{t})\bmid \initparam}}
    \sqrt{\frac{1}{n} \sum_{i=1}^{n} \E\br{ \|\bW_{t} - \bW_{t}\repi\|_{\text{op}}^2\bmid \initparam}}\\
    &+ C_x B_{\phi^{\prime}}^2 \cdot
    \frac{1}{n} \sum_{i=1}^{n} \E\br{ \|\bW_{t}-\bW_{t}\repi\|_{\text{op}}^2\bmid \initparam}
\end{align*}
where $\|\cdot\|_{\text{op}}$ denotes the spectral norm. 
\end{nameddef}
We note while the stability is only required on the spectral norm, our bound will be on the element wise $L_2$-norm i.e. Frobenius norm, which upper bounds the spectral norm. It is summarised within the following lemma shown in \cref{sec:lem:NNL2Gen}.
\begin{lemma}[Bound on On-Average Parameter Stability]
\label{lem:NNL2Gen}
Consider Assumptions \ref{ass:Activation} and \ref{ass:boundedness}.
Fix $t > 0$. If $\eta \leq 1 / (2  \rho) $, then
\begin{align*}
    \frac{1}{n} \sum_{i=1}^{n} \E\br{ \|\bW_{t+1} - \bW_{t+1}\repi\|_F^2\bmid \initparam}
    \leq 
  8 e\frac{\eta^2 t }{n^2}  
    \Big(\frac{1}{1 - 2 \eta \epsilon} \Big)^t
    \frac{1}{n} \sum_{i=1}^{n}
    \sum_{j=0}^{t} \E\br{\|\nabla \ell(\bW_{j},z_i)\|_2^2\bmid \initparam} 
\end{align*}
where $\epsilon = 2 \cdot \frac{C_x^2 \sqrt{C_0} B_{\phi^{\prime\prime}} }{\sqrt{m}} \big( 4 B_{\phi^{\prime}} C_x \sqrt{ \eta t} + \sqrt{2}\big)$.
\end{lemma}
Theorem \ref{thm:gen_gap} then arises by combining Lemma \ref{lem:NNGenL2} and Lemma \ref{lem:NNL2Gen}, and noting the following three points. Firstly, recall that 
\begin{align*}
    \frac{1}{n} \sum_{i=1}^{n} \|\nabla \ell(\bW,z_i)\|_2^2  
    & \leq \big(\max_{i=1,\dots,n} \|\nabla f_{\bW}(\bx_i)\|_2^2 \big) \frac{1}{n} \sum_{i=1}^{n} (f_{\bW}(\bx_i) - y_i)^2\\ & \leq 2 C_x^2 B_{\phi^{\prime}}^2 \sL_{S}(\bW).
\end{align*}
Secondly, note that we have $\big(\frac{1}{1 - 2 \eta \epsilon} \big)^t \leq \exp( \frac{2 \eta t \epsilon}{1- 2 \eta t \epsilon}) \leq e^2 $ when $2 \eta t \epsilon \leq 2/3$. For this to occur we then require
\begin{align*}
    \epsilon =  2 \cdot \frac{C_x^2 \sqrt{C_0} B_{\phi^{\prime\prime}} }{\sqrt{m}} \cdot \big( 4 B_{\phi^{\prime}} C_x \sqrt{ \eta t} + \sqrt{2}\big) \leq \frac{1}{3 \eta t},
\end{align*}
which is satisfied by scaling $m$ sufficient large, in particular, as required within condition \eqref{equ:sufficientwidth} within the statement of Theorem \ref{thm:gen_gap}. This allows us to arrive at the bound on the $L_2$-stability 
\begin{align*}
    \frac{1}{n} \sum_{i=1}^{n} \E\br{ \|\bW_{t+1} - \bW_{t+1}\repi\|_F^2 \bmid \initparam}
    \leq
    \frac{\eta^2 t}{n^2} \cdot 16 e^3 C_x^2 B_{\phi^{\prime}}^2 \sum_{j=0}^{t} \E\br{\sL_{S}(\bW_j)\bmid \initparam}~.
\end{align*}
Third and finally,  
note that we can bound
\begin{align*}
  &\sqrt{\E\br{\sL_{S}(\bW_{t+1}) \bmid \initparam}}
    \sqrt{ \frac{\eta^2 t }{n^2}  
    \sum_{j=0}^{t} \E\br{\sL_{S}(\bW_{j}) \bmid \initparam}
    }\\
    & = 
    \frac{\eta }{n} 
    \sqrt{t\E\br{\sL_{S}(\bW_{t+1}) \bmid \initparam}}
    \sqrt{
    \sum_{j=0}^{t} \E\br{\sL_{S}(\bW_{j}) \bmid \initparam}
    } \\
    & \leq 
    \frac{\eta }{n} 
    \sum_{j=0}^{t} \E\br{\sL_{S}(\bW_{j}) \bmid \initparam}
\end{align*}
since $\sL_{S}(\bW_{t+1}) \leq 
\frac{1}{t} \sum_{j=1}^{t}\sL_{S}(\bW_{j}) $. This then results in 
\begin{align*}
    &\E\br{\sL(\bW_{t+1}) - \sL_{S}(\bW_{t+1})\bmid \initparam}\\
    & \leq 
    \Big( \frac{\eta}{n} \big(4 e^2 C_x^{3/2}B_{\phi^{\prime}}^{2}\big) + 
    \frac{\eta^2 t}{n^2} \big(16 e^3 C_x^{3} B_{\phi^{\prime}}^4\big)
    \Big)
    \sum_{j=0}^{t}\E\br{\sL_{S}(\bW_j)\bmid \initparam}\\
    & \leq 
    16 e^3 C_x^{3/2}B_{\phi^{\prime}}^2(1 + C_x^{3/2}B_{\phi^{\prime}}^2)
    \Big( \frac{\eta}{n} + \frac{\eta^2 t}{n^2}\Big)
    \sum_{j=0}^{t}\E\br{\sL_{S}(\bW_j)\bmid \initparam}
\end{align*}
as required.

\subsection{Proof of Lemma \ref{lem:NNGenL2}: From loss stability to parameter stability}
\label{sec:lem:NNGenL2}
Recall that $\widetilde{z}_i = (\widetilde{\bx}_i, y_i) \in \sB_2^d(C_x) \times [-C_y, C_y]$.
Expanding the square loss and some basic algebra gives us:
\begin{align*}
    & 2 \pr{ \ell(\bW_{t},\widetilde{z}_i) - \ell(\bW_{t}\repi,\widetilde{z}_i)} \\
    & = 
      \pr{f_{\bW_{t}}(\widetilde{\bx}_i) - \widetilde{y}_i}^2 - 
     \pr{f_{\bW_{t}\repi}(\widetilde{\bx}_i) - \widetilde{y}_i}^2 \\
    & = 
    \pr{f_{\bW_{t}}(\widetilde{\bx}_i) - \widetilde{y}_i}
      \pr{f_{\bW_{t}}(\widetilde{\bx}_i) - f_{\bW_{t}\repi}(\widetilde{\bx}_i)}
    + 
      \pr{f_{\bW_{t}\repi}(\widetilde{\bx}_i) - \widetilde{y}_i}
      \pr{f_{\bW_{t}}(\widetilde{\bx}_i)- f_{\bW_{t}\repi}(\widetilde{\bx}_i) }\\
    & = 
    \pr{f_{\bW_{t}}(\widetilde{\bx}_i)- f_{\bW_{t}\repi}(\widetilde{\bx}_i) }^2 
    + 
    2 \pr{f_{\bW_{t}\repi}(\widetilde{\bx}_i) - \widetilde{y}_i}
    \pr{f_{\bW_{t}}(\widetilde{\bx}_i)- f_{\bW_{t}\repi}(\widetilde{\bx}_i) }~.
\end{align*}
We then have 
\begin{align*}
    & \frac{1}{n} \sum_{i=1}^{n} \E\br{ \ell(\bW_{t},\widetilde{z}_i) - \ell(\bW_{t}\repi,\widetilde{z}_i) \bmid \initparam} \\
    & \leq
    \frac{1}{n} \sum_{i=1}^{n}
    \E\br{ 
    \abs{\pr{f_{\bW_{t}\repi}(\widetilde{\bx}_i) - \widetilde{y}_i}}
    \abs{\pr{f_{\bW_{t}}(\widetilde{\bx}_i)- f_{\bW_{t}\repi}(\widetilde{\bx}_i) }}\bmid \initparam}\\
    &\quad\quad 
    +
    \frac{1}{2n} \sum_{i=1}^{n} \E\br{ \pr{f_{\bW_{t}}(\widetilde{\bx}_i)- f_{\bW_{t}\repi}(\widetilde{\bx}_i) }^2 \bmid \initparam}\\
    & \leq 
     \sqrt{ \frac{1}{n} \sum_{i=1}^{n} \E\br{\pr{f_{\bW_{t}\repi}(\widetilde{\bx}_i) - \widetilde{y}_i}^2\bmid \initparam}}
    \sqrt{ 
    \frac{1}{n} \sum_{i=1}^{n} \E\br{ \pr{f_{\bW_{t}}(\widetilde{\bx}_i)- f_{\bW_{t}\repi}(\widetilde{\bx}_i) }^2 \bmid \initparam}}
    \\
    &\quad\quad + 
    \frac{1}{2n} \sum_{i=1}^{n} \E\br{ \pr{f_{\bW_{t}}(\widetilde{\bx}_i)- f_{\bW_{t}\repi}(\widetilde{\bx}_i) }^2 \bmid \initparam} 
\end{align*}
where performing steps as in \cref{eq:f_diff_1}-\eqref{eq:f_diff_2} we have
\begin{align*}
  \pr{f_{\bW_{t}}(\widetilde{\bx}_i)- f_{\bW_{t}\repi}(\widetilde{\bx}_i) }^2
  \leq
  C^2_x B_{\phi^{\prime}}^2 
  \|\bW_{t} - \bW_{t}\repi\|_2^2~.
\end{align*}
Plugging in this bound then yields the result.

\subsection{Proof of Lemma \ref{lem:NNL2Gen}: Bound on on-average parameter stability}
\label{sec:lem:NNL2Gen}
Throughout the proof empirical risk w.r.t.\ remove-one tuple $S\deli$ is denoted as
\[
  \sL_{S\deli}(\bW) = \sL_{S}(\bW) - \frac{1}{n} \ell(\bW,z_i) = \sL_{S_i}(\bW) - \frac{1}{n} \ell(\bW,\widetilde{z}_i)~.
\]
Plugging in the gradient updates with the inequality $(a+b)^2 \leq (1+p)a^2 + (1+1/p)b^2$ for $p > 0$   then yields (this technique having been applied within \citep{lei2020fine})
\begin{align*}
    \|\bW_{t+1} - \bW_{t+1}\repi\|_2^2 
    & \leq 
    (1+p) \underbrace{\|\bW_{t} - \bW_{t}\repi - \eta \pr{ \nabla \sL_{S\deli}(\bW_t) - \nabla \sL_{S\deli}(\bW_t\repi)} \|_2^2}_{\text{Expansiveness of the Gradient Update}} \\
    &
    + 
    (1+1/p) \cdot \frac{2 \eta^2}{n^2} \cdot \pr{ \|\nabla \ell(\bW_t,z_i)\|_2^2 + \|\nabla \ell(\bW_t\repi),\widetilde{z}_i)\|_2^2 }~.
\end{align*}

We must now bound the expansiveness of the gradient update.
Opening-up the squared norm we get
\begin{align*}
  &\|\bW_{t} - \bW_{t}\repi - \eta ( \nabla \sL_{S\deli}(\bW_t) - \nabla \sL_{S\deli}(\bW_t\repi)) \|_2^2\\
  &=
  \|\bW_{t} - \bW_{t}\repi\|^2_2 + \eta^2 \|\nabla \sL_{S\deli}(\bW_t) - \nabla \sL_{S\deli}(\bW_t\repi) \|_2^2\\
  &\quad\quad
    - 2 \eta \ip{\bW_{t} - \bW_{t}\repi, \nabla \sL_{S\deli}(\bW_t) - \nabla \sL_{S\deli}(\bW_t\repi)}
\end{align*}
For this purpose we will use the following key lemma shown in \cref{sec:lem:Coecer_proof}:
\begin{lemma}[Almost Co-coercivity of the Gradient Operator]
\label{lem:Coecer}
Consider the assumptions of Lemma \ref{lem:NNL2Gen}. Then for $t \geq 1$
\begin{align*}
    & \langle \bW_{t} - \bW_{t}\repi , \nabla \sL_{S\deli}(\bW_t) - \nabla \sL_{S\deli}(\bW_t\repi)
    \rangle 
     \geq 
    2 \eta \pr{ 1 - \frac{\eta \rho}{2}} \|\nabla \sL_{S\deli}(\bW_t) - \nabla \sL_{S\deli}(\bW_t\repi)\|_2^2 \\
    & \quad\quad\quad\quad\quad\quad 
    - \epsilon 
    \lf\|\bW_{t} - \bW_{t}\repi - \eta\pr{ 
    \nabla \sL_{S\deli}(\bW_t) - \nabla \sL_{S\deli}(\bW_t\repi)
    } \rt\|_2^2
\end{align*}
where 
\begin{align*}
    \rho & = 
    C_x^2 \pr{ 
    B_{\phi^{\prime}}^2
    +B_{\phi^{\prime\prime}}B_{\phi} +\frac{B_{\phi^{\prime\prime}} C_y}{\sqrt{m}}
    }~,\\
    \epsilon & = 
    2 \cdot \frac{C_x^2 \sqrt{C_0} B_{\phi^{\prime\prime}}}{\sqrt{m}} \pr{ 4 B_{\phi^{\prime}} C_x \sqrt{ \eta t} + \sqrt{2}}~.
\end{align*}
\end{lemma}

Thus by \cref{lem:Coecer} we get
\begin{align*}
    & \|\bW_{t} - \bW_{t}\repi - \eta \pr{ \nabla \sL_{S\deli}(\bW_t) - \nabla \sL_{S\deli}(\bW_t\repi)} \|_2^2\\
    & \leq  
    \|\bW_{t} - \bW_{t}\repi\|_2^2 
    + \eta^2(2 \eta \rho- 3) \lf\|\nabla \sL_{S\deli}(\bW_t) - \nabla \sL_{S\deli}(\bW_t\repi) \rt\|_2^2 \\
    & \quad\quad 
    + 2 \eta \epsilon 
    \lf\|\bW_{t} - \bW_{t}\repi - \eta\pr{ 
    \nabla \sL_{S\deli}(\bW_t) - \nabla \sL_{S\deli}(\bW_t\repi)
    } \rt\|_2^2.
\end{align*}
Rearranging  and using that $\eta \rho \leq 1/2$ we then arrive at the recursion 
\begin{align*}
  \|\bW_{t+1} - \bW_{t+1}\repi\|_F^2 
   &\leq 
    \frac{1 + p}{1 - 2 \eta \epsilon}
    \cdot
    \|\bW_{t} - \bW_{t}\repi\|_F^2\\
    &\quad\quad + 
    \pr{1+\frac{1}{p}} \cdot \frac{2 \eta^2  }{n^2}
    \pr{ \|\nabla \ell(\bW_t,z_i)\|_2^2 + \|\nabla \ell(\bW_t\repi),\widetilde{z}_i)\|_2^2} \\
  &\leq 
    \pr{1+\frac{1}{p}} \cdot \frac{2 \eta^2 }{n^2}
    \pr{\frac{1 + p}{1 - 2 \eta \epsilon} }^t
    \sum_{j=0}^{t} 
    \pr{ \|\nabla \ell(\bW_j,z_i)\|_2^2 + \|\nabla \ell(\bW_j\repi),\widetilde{z}_i)\|_2^2}~.
\end{align*}
Taking expectation and summing we then get 
\begin{align*}
  &\frac{1}{n}\sum_{i=1}^{n}\E\br{ \|\bW_{t+1} - \bW_{t+1}\repi\|_F^2\bmid \initparam}\\
    &\qquad\qquad\leq 
    4 (1+1/p) \frac{2 \eta^2 }{n^2}  
    \pr{\frac{(1 + p)}{1 - 2 \eta \epsilon} }^t
    \sum_{j=0}^{t} \E\br{ \|\nabla \ell(\bW_j,z_i)\|_2^2\bmid \initparam}
\end{align*}
where we note that $\E\br{\|\nabla \ell(\bW_j,z_i)\|_2^2\bmid \initparam} = \E\br{\|\nabla\ell(\bW_{j}\repi,\widetilde{z}_i)\|_2^2\bmid \initparam}$ since $z_i$ and $\widetilde{z}_i$ are identically distributed.
Picking $p = 1/t$ and noting that $(1+p)^{t} = (1+1/t)^t \leq e$ yields the bound.

\subsection{Proof of \cref{lem:Coecer}: Almost-co-coercivity of the Gradient Operator}
In this section we show \cref{lem:Coecer} which says that a gradient operator of an overparameterised shallow network is almost-co-coercive.
The proof of this lemma will require two auxiliary lemmas.
\begin{lemma}
\label{lem:Useful}
Consider Assumptions \ref{ass:Activation} and \ref{ass:boundedness} and assume that $\eta \leq 1/(2 \rho)$.
Then for any $t \geq 0 $, $i \in [n]$,
\begin{align*}
    \|\bW_{t} - \bW_0\|_F &\leq 
    \sqrt{2 \eta t \sL_{S}(\bW_{0})}~, \\
    \|\bW_{t}\repi - \bW_0\|_F 
  &\leq 
     \sqrt{2 \eta t  \sL_{S\repi}(\bW_{0})}~.
\end{align*}
\end{lemma}
\begin{proof}
  The proof is given in \cref{sec:lem:Useful_proof}.
\end{proof}
We also need the following Lemma (whose proof is very similar to \cref{lem:eigenvalues}).
\begin{lemma}
  \label{lem:Useful2}
  Consider Assumptions \ref{ass:Activation} and \ref{ass:boundedness}.
  Fix $s \geq 0 $, $i \in [n]$.
  For any $\alpha \in [0,1]$ denote
\begin{align*}
    \bW(\alpha) & \df
    \bW_s\repi 
    + 
    \alpha \pr{ \bW_s - \bW_{s}\repi - \eta \pr{ \nabla \sL_{S\deli}(\bW_s) - \nabla \sL_{S\deli}(\bW_s\repi) }
    }~,\\
    \widetilde{\bW}(\alpha)
    & \df
    \bW_s + 
    \alpha 
      \pr{ \bW_s\repi - \bW_{s} - \eta \pr{ \nabla \sL_{S\deli}(\bW_s\repi) - \nabla \sL_{S\deli}(\bW_s) } }~.
\end{align*}
If $\eta \leq 1/(2 \rho)$, then
\begin{align*}
    \min_{\alpha \in [0,1]} 
    \lambda_{\min}\pr{\nabla^2 \sL_{S\deli}\pr{\bW(\alpha)} }
     \geq 
     - \widetilde{\epsilon}~, \\
     \min_{\alpha \in [0,1]} 
     \lambda_{\min}\pr{\nabla^2 \sL_{S\deli}\pr{\widetilde{\bW}(\alpha)}}
     \geq - \widetilde{\epsilon}~.
\end{align*}
with 
\begin{align*}
  \widetilde{\epsilon} = \frac{C_x^2 B_{\phi^{\prime\prime}}}{\sqrt{m}} \pr{ 4 B_{\phi^{\prime}} C \sqrt{ \eta s} \pr{ \sqrt{\sL_S(\bW_0)} + \sqrt{\sL_{S_i}(\bW_0)} } + \sqrt{2 \sL_{S_{i}}(\bW_0)} + \sqrt{2 \sL_{S}(\bW_0)}
  }~.
\end{align*}
\end{lemma}
\begin{proof}
  The proof is given in \cref{sec:lem:Useful2_proof}.
\end{proof}
\subsubsection{Proof of Lemma \ref{lem:Coecer}}
\label{sec:lem:Coecer_proof}
The proof of this Lemma follows by arguing that the operator $\bw \mapsto \nabla \sL_{S\deli}(\bw)$ is almost-co-coercive:
Recall that the operator $F ~:~ \sX \to \sX$ is co-coercive whenever
$\ip{\nabla F(\bx) - \nabla F(\by), \bx - \by} \geq \alpha \|\nabla F(\bx) - \nabla F(\by)\|^2$
holds for any $\bx, \by \in \sX$ with parameter $\alpha > 0$.
In our case, right side of the inequality will be replaced by $\alpha \|\nabla F(\bx) - \nabla F(\by)\|^2 - \ve$, where $\ve$ is a small.

Let us begin by defining the following two functions 
\begin{align*}
    \psi(\bW) = \sL_{S\deli }(\bW) - \langle \nabla \sL_{S\deli}(\bW\repi_t), \bW\rangle~,
    \quad\quad 
    \psi^{\star}(\bW) = \sL_{S\deli}(\bW) - \langle \nabla \sL_{S\deli}(\bW_{t}), \bW\rangle~.
\end{align*}
Observe that
\begin{align}
  \label{eq:mono_eq}
   & \psi(\bW_t)  - \psi(\bW_t\repi)   + 
    \psi^{\star}(\bW_t\repi) - \psi^{\star}(\bW_t)\\
    & = 
    \sL_{S\deli }(\bW_t) - \langle \nabla \sL_{S\deli}(\bW\repi_t), \bW_t\rangle
    - 
    \sL_{S\deli }(\bW_t\repi) 
    +\langle \nabla \sL_{S\deli}(\bW\repi_t), \bW_t\repi\rangle \nonumber\\
    & 
    + 
    \sL_{S\deli }(\bW_t\repi)
    - \langle \nabla \sL_{S\deli}(\bW_t),\bW_{t}\repi\rangle
    - 
    \sL_{S\deli }(\bW_t)
      + 
    \langle \nabla \sL_{S\deli}(\bW_t),\bW_{t}\rangle \nonumber\\
    & = \langle \bW_t - \bW_{t}\repi, \nabla \sL_{S\deli}(\bW_t) - \nabla \sL_{S\deli}(\bW_t\repi)\rangle~, \nonumber
\end{align}
from which follows that we are interesting in giving lower bounds on $\psi(\bW_t)  - \psi(\bW_t\repi)$ and $\psi^{\star}(\bW_t\repi) - \psi^{\star}(\bW_t)$.

From Lemma \ref{lem:eigenvalues} we know the loss is $\rho$-smooth with $\rho = C_x^2  \pr{B_{\phi^{\prime}}^2 + B_{\phi^{\prime\prime}}B_{\phi} + \frac{C_yB_{\phi^{\prime\prime}} }{\sqrt{m}}}$, and thus, for any $i \in [n]$, we immediately have the upper bounds
\begin{align}
  \label{eq:psi_upper_bounds_1}
    \psi(\bW_{t} - \nabla \psi(\bW_t) )
  &\leq 
    \psi(\bW_t) - \eta\pr{ 1 - \frac{\eta \rho}{2}} \|\nabla \psi(\bW_t)\|_2^2 \\
    \psi^{\star}(\bW_{t}\repi- \eta \nabla \psi^{\star}(\bW_{t}\repi))
  &\leq  
    \psi^{\star}(\bW_t\repi) - \eta \pr{ 1 - \frac{\eta \rho}{2} } \|\nabla \psi^{\star}(\bW_t\repi)\|_2^2  \label{eq:psi_upper_bounds_2}
\end{align}
Now, in the smooth and convex case  \citep{nesterov2003introductory}, convexity would be used here to lower bound the left side of each of the inequalities by $\psi(\bW_t^{(i)})$ and $\psi^{\star}(\bW_t)$ respectively. In our case, while the functions are not convex, we can get an ``approximate'' lower bound by leveraging that the minimum Eigenvalue evaluated at the points $\bW_t,\bW_t^{(i)}$ is not too small. More precisely, we have the following lower bounds by applying Lemma \ref{lem:Useful2},
which will be shown shortly:
\begin{align}
    \psi(\bW_{t} - \eta \nabla \psi(\bW_t) )
    & \geq 
    \psi(\bW_t\repi) 
    - \frac{\epsilon}{2} \|\bW_{t} -\bW_t\repi - \eta \nabla \psi(\bW_t)\|_2^2~, \label{eq:psi_lower_bounds_1}\\
    \psi^{\star}(\bW_{t}\repi - \eta \nabla \psi^{\star}(\bW\repi_t) )
    & \geq 
      \psi^{\star}(\bW_t)   - \frac{\epsilon}{2} \|\bW_t\repi - \bW_{t} - \eta \nabla \psi^{\star}(\bW_t\repi)\|_2^2. \label{eq:psi_lower_bounds_2}
\end{align}
Combining this with \cref{eq:psi_upper_bounds_1}, \eqref{eq:psi_upper_bounds_2}, and rearranging we get:
\begin{align}
    \psi(\bW_t)  - \psi(\bW_t\repi)  
    & \geq 
    \eta\pr{ 1 - \frac{\eta \rho}{2}} \|\nabla \psi(\bW_t)\|_2^2 
    - 
    \frac{\epsilon}{2} \|\bW_{t} -\bW_t\repi - \eta \nabla \psi(\bW_t)\|_2^2~, \\
    \psi^{\star}(\bW_t\repi) - \psi^{\star}(\bW_t)
    & \geq 
    \eta\pr{ 1 - \frac{\eta \rho}{2}} \|\nabla \psi^{\star}(\bW_t\repi)\|_2^2
    - 
    \frac{\epsilon}{2} \|\bW_t\repi - \bW_{t} - \eta \nabla \psi^{\star}(\bW_t\repi)\|_2^2.
\end{align}
Adding together the two bounds and plugging into \cref{eq:mono_eq} completes the proof.

\paragraph{Proof of \cref{eq:psi_lower_bounds_1} and \cref{eq:psi_lower_bounds_2}.} All that is left to do, is to prove \cref{eq:psi_lower_bounds_1} and \eqref{eq:psi_lower_bounds_2}.
To do that, we will use \cref{lem:Useful2} while recalling the definition of  $\bW(\alpha)$ and $\widetilde{\bW}(\alpha)$ given in the Lemma.
That said, let us then define the following two functions:
\begin{align*}
    g(\alpha)
    \df
    \psi(\bW(\alpha)) 
    + \frac{\widetilde{\epsilon} \alpha^2}{2}
    \|\bW_{t} -\bW_t\repi - \eta \big( \nabla \sL_{S\deli}(\bW_t) - \nabla \sL_{S\deli}(\bW_t\repi)\big)\|_2^2~, \\
    \widetilde{g}(\alpha) 
    \df
    \psi^{\star}(\widetilde{\bW}(\alpha))
    + 
    \frac{\widetilde{\epsilon} \alpha^2}{2}
    \|\bW_{t} -\bW_t\repi - \eta \big( \nabla \sL_{S\deli}(\bW_t) - \nabla \sL_{S\deli}(\bW_t\repi)\big)\|_2^2~.
\end{align*}
Note that from Lemma \ref{lem:Useful2} we have that $g^{\prime\prime}(\alpha),\widetilde{g}^{\prime\prime}(\alpha) \geq 0$ for $\alpha \in [0,1]$. Indeed, we have with $\Delta \df \bW_{t} -\bW_t\repi - \eta \big( \nabla \sL_{S\deli}(\bW_t) - \nabla \sL_{S\deli}(\bW_t\repi)\big)$:
\begin{align*}
    g^{\prime\prime}(\alpha) 
    = 
    \ip{\Delta, \nabla^{2} \sL_{S\deli}(\bW(\alpha))
    \Delta}
    + 
    \widetilde{\epsilon} \|\Delta\|_2^2 
    \geq 0
\end{align*}
and similarly for $\widetilde{g}(\alpha)$. Therefore both $g(\cdot)$ and $\widetilde{g}(\cdot)$ are convex on $[0,1]$. The first inequality then arises from noting the follow three points. Since $g$ is convex we have $g(1) - g(0) \geq g^{\prime}(0)$ with $g^{\prime}(0) = 
\langle \nabla \psi(\bW_t\repi), \Delta \rangle = 0$ since $\nabla \psi(\bW_t\repi) = 0$. This yields 
\begin{align*}
    0 & \leq g(1) - g(0) \\
    & = 
    \psi(\bW_t - \eta \nabla \psi(\bW_t)) 
    + 
    \frac{\widetilde{\epsilon}}{2} \|\bW_{t} -\bW_t\repi - \eta \big( \nabla \sL_{S\deli}(\bW_t) - \nabla \sL_{S\deli}(\bW_t\repi)\big)\|_2^2
    - 
    \psi(\bW_t\repi) 
\end{align*}
which is almost \cref{eq:psi_lower_bounds_1}: The missing step is showing that $\widetilde{\epsilon} \leq \epsilon$.
This comes by the uniform boundedness of the loss, that is, having $\ell(\bW_0,z) \leq C_0$ a.s.\ we can upper-bound
\begin{align*}
    \widetilde{\epsilon} \leq 
    2 \cdot \frac{C_x^2 \sqrt{C_0} B_{\phi^{\prime\prime}} }{\sqrt{m}} 
    \big( 4 B_{\phi^{\prime}} C_x \sqrt{\eta s} + \sqrt{2} \big)
    = \epsilon
\end{align*}
This proves \cref{eq:psi_lower_bounds_1}, while \cref{eq:psi_lower_bounds_2} comes by following similar steps and considering $\widetilde{g}(1) - \widetilde{g}(0) \geq \widetilde{g}^{\prime}(0)$.

\subsubsection{ Proof of Lemma \ref{lem:Useful}}
\label{sec:lem:Useful_proof}
Recalling the Hessian \eqref{equ:LossHessian} we have for any parameter $\bW$ and data point $z = (\bx, y)$,
\begin{align*}
    \|\nabla^2 \ell(\bW,z) \|_2
    & \leq 
    \|\nabla f_{\bW}(\bx)\|_2^2 
    + 
    \|\nabla^2 f_{\bW}(\bx)\|_2 |f_{\bW}(\bx) - y|\\
    & \leq 
    C_x^2 \pr{ 
    B_{\phi^{\prime}}^2
    +B_{\phi^{\prime\prime}}B_{\phi} +\frac{B_{\phi^{\prime\prime}} C_y}{\sqrt{m}}
    }
\end{align*}
That is we have from \eqref{equ:HessianBound} the bound $\|\nabla^2 f_{\bW}(\bx)\|_2  \leq \frac{C^2_x }{\sqrt{m}}  B_{\phi^{\prime\prime}}$, meanwhile we can trivially bound 
\begin{align*}
    |f_{\bW}(\bx) - y|
    & \leq 
    \frac{1}{\sqrt{m}} \sum_{j=1}^{m} |\phi(\langle (\bW)_j,\bx \rangle)|  
    + C_y \\
    & \leq 
    \sqrt{m} B_{\phi} + C_y.
\end{align*}
and 
\begin{align*}
    \|\nabla f_{\bW}(\bx)\|_2^2
    & = 
    \|\bx\|_2^2 \cdot \frac{1}{m} \sum_{j=1}^{m} \phi^{\prime}(\langle (\bW)_j,\bx\rangle)\\
    & \leq 
    C_x^2  B_{\phi^{\prime}}^2.
\end{align*}
The loss is therefore $\rho$-smooth with $\rho = C_x^2 \pr{ B_{\phi^{\prime}}^2+B_{\phi}B_{\phi^{\prime\prime}} +\frac{C_y B_{\phi^{\prime\prime}}}{\sqrt{m}}}$. Following standard arguments we then have for $ j \in \mathbb{N}_0 $
\begin{align*}
    \sL_{S}(\bW_{j+1}) 
    \leq 
    \sL_{S}(\bW_{j}) 
    - \eta \pr{1 - \frac{\eta \rho}{2} }
    \|\nabla \sL_{S}(\bW_j)\|_F^2
\end{align*}
which when rearranged and summed over $j$ yields 
\begin{align*}
    \eta\big(1 - \frac{\eta \rho}{2}\big)  
    \sum_{j=0}^{t}
    \|\nabla \sL_{S}(\bW_j)\|_F^2
    \leq 
    \sum_{j=0}^{t}\sL_{S}(\bW_{j})  - \sL_{S}(\bW_{j+1}) 
    = 
    \sL_{S}(\bW_0) - \sL_{S}(\bW_{t+1})
\end{align*}
We also note that 
\begin{align*}
    \bW_{t+1} - \bW_{0}
    = 
    - \eta \sum_{s=0}^{t} \nabla \sL_{S}(\bW_{s})
\end{align*}
and therefore  by convexity of the squared norm we have $\|\bW_{t+1} - \bW_{0}\|_F^2
= \eta^2 \|\sum_{s=0}^{t} \nabla \sL_{S}(\bW_{s})\|_F^2 \leq \eta^2t \sum_{s=0}^{t} \|\nabla \sL_{S}(\bW_{s})\|_F^2 $. Plugging this in we get  when $\eta \rho \leq 1/2$
\begin{align*}
    \frac{3}{4} \cdot \frac{1}{\eta t} \|\bW_{t+1} - \bW_{0}\|_F^2
    \leq 
    \sL_{S}(\bW_0)
\end{align*}
Rearranging then yields the inequality. An identical set of steps can be performed for the cases $\bW_{t}\repi$ for $i \in [n]$.

\subsubsection{Proof of Lemma \ref{lem:Useful2}}
\label{sec:lem:Useful2_proof}
Looking at \eqref{equ:LossHessian} we note the first matrix is positive semi-definite and therefore for any $\bW \in \mathbb{R}^{d m}$:
\begin{align*}
    \lambda_{\min}( \nabla^2 \sL_{S\deli}(\bW) ) 
    & \geq  
    - \lambda_{\max}\pr{
    \frac{1}{n} \sum_{j \in [n] : j \neq i} \nabla^2 f_{\bW}(\bx_i) \big( f_{\bW}(\bx_j) - y_j\big)
    } \\
    & \geq 
    - \frac{C_x^2 B_{\phi^{\prime\prime}}}{\sqrt{m}} \cdot
    \frac{1}{n} \sum_{j \in [n] : j \neq i}  | f_{\bW}(x_j) - y_j|
\end{align*}
where we have used the operator norm of the Hessian $\|\nabla^2 f_{\bW}(\bx)\|_2$ bound \eqref{equ:HessianBound}.
We now choose $\bW = \bW(\alpha)$ and thus need to bound $\frac{1}{n} \sum_{j \in [n] : j \neq i}  |f_{\bW(\alpha)}(x_j) - y_i|$ and $\frac{1}{n} \sum_{j \in [n] : j \neq i}  |f_{\widetilde{\bW}(\alpha)}(\bx_i) - y_i|$. Note that we then have for any iterate $\bW_t$ with $t \in \mathbb{N}_0$,
\begin{align*}
    \frac{1}{n} \sum_{j \in [n] : j \neq i}  |f_{\bW(\alpha)}(\bx_i) - y_i| 
    & \leq 
    \frac{1}{n} \sum_{j \in [n] : j \neq i}  |f_{\bW(\alpha)}(\bx_i) - f_{\bW_t\repi}(\bx_i)|
    + 
    \frac{1}{n} \sum_{j \in [n] : j \neq i}  | f_{\bW_t\repi}(\bx_i) - y_i| \\
    & \leq 
    B_{\phi^{\prime}} C_x  
    \|\bW(\alpha) - \bW_t\repi\|_F
    +
      \sqrt{2 \sL_{S\deli}(\bW_t\repi)}
\end{align*}
where the first term on the r.h.s.\ is bounded using Cauchy-Schwarz inequality as in \cref{eq:f_diff_1}-\eqref{eq:f_diff_2},
and the second term is bounded by Jensen's inequality.
A similar calculation yields 
\begin{align*}
    \frac{1}{n} \sum_{j=1, j\not= i}^{n}  |f_{\widetilde{\bW}(\alpha)}(\bx_i) - y_i| 
    \leq
    B_{\phi^{\prime}} C_{x}
    \|\widetilde{\bW}(\alpha) - \bW_t\|_F
    + 
  \sqrt{2 \sL_{S\deli}(\bW_t\repi)}~.
\end{align*}
Since the loss is $\rho$-smooth by \cref{lem:eigenvalues} we then have 
\begin{align*}
    \|\bW(\alpha) - \bW_t\repi\|_F
   &  \leq 
   \alpha\big( 
    \|\bW_t - \bW_t\repi\|_F 
    + \eta \|\nabla \sL_{S\deli}(\bW_t) - \nabla \sL_{S\deli}(\bW_t\repi)\|_F
    \big) \\
    & \leq 
    (1+\eta \rho)
    \|\bW_t - \bW_t\repi\|_F \\
    & \leq 
    \frac{3}{2}
    \big( \|\bW_t - \bW_0 \|_F + \|\bW_0 - \bW_t\repi\|_F\big) \\
    & \leq
    \frac{3}{2} \sqrt{2 \eta s}\big(\sqrt{\sL_{S}(\bW_0)} + \sqrt{\sL_{S\repi}(\bW_0)}\big)
\end{align*}
where at the end we used Lemma \ref{lem:Useful}. A similar calculation yields the same bound for $ \|\widetilde{\bW}(\alpha) - \bW_t\|_F$. Bringing together we get 
\begin{align*}
    \lambda_{\min}( \nabla^2 \sL_{S\deli}(\bW(\alpha)) ) 
    & \geq
    - \frac{C_x^2 B_{\phi^{\prime\prime}}}{\sqrt{m}} 
    \pr{ 4 B_{\phi^{\prime}} C_{x} \sqrt{ \eta s}\big(\sqrt{\sL_{S}(\bW_0)} + \sqrt{\sL_{S\repi}(\bW_0)}\big)  +  \sqrt{2 \sL_{S\deli}(\bW_t\repi)}
    }\\
    \lambda_{\min}( \nabla^2 \sL_{S\deli}(\widetilde{\bW}(\alpha)) ) 
    & \geq 
    - \frac{C_x^2 B_{\phi^{\prime\prime}}}{\sqrt{m}} 
      \pr{ 4 B_{\phi^{\prime}} C_{x} \sqrt{\eta s}\big(\sqrt{\sL_{S}(\bW_0)} + \sqrt{\sL_{S\repi}(\bW_0)}\big)  +  \sqrt{2 \sL_{S\deli}(\bW_t)}
      }
\end{align*}
The final bound arises from noting that $\sL_{S\deli}(\bW_t) \leq \sL_{S}(\bW_t) \leq \sL_{S}(\bW_0)$ and $\sL_{S\deli}(\bW_t\repi) \leq \sL_{S\repi}(\bW\repi_t) \leq \sL_{S\repi}(\bW_0)$. 

\section{Connection between $\oracle$ and \ac{NTK}}
\label{sec:W_W0_NTK}
This section is dedicated to the proof of \cref{lem:oracle_NTK}.
We will first need the following standard facts about the \ac{NTK}.
\begin{lemma}[NTK Lemma]
  \label{lem:NTK}
  For any $\bW, \btilW \in \R^{d \times m}$ and any $\bx \in \R^d$,
  \begin{align*}
    f_{\bW}(\bx)
    =
    f_{\btilW}(\bx)
    +
    \sum_{k=1}^m u_k \phi'\pr{\ip{\bx, \btilW_{k}}} \ip{\bW_k - \btilW_{k}, \bx}    
    +
    \eps(\bx)
  \end{align*}
  where
  \begin{align*}
    \eps(\bx) =
    \frac12 \sum_{k=1}^m u_k \pr{\int_0^1 \phi''\pr{\tau \ip{\bx, \bW_k} + (1-\tau) \ip{\bx, \btilW_{k}}} \diff \tau} \ip{\bx, \bW_k - \btilW_{k}}^2~.
  \end{align*}

  Note that
  \begin{align*}
    |\eps(\bx)| \leq \frac{B_{\phi''} \|\bx\|}{2 \sqrt{m}} \cdot \|\bW - \btilW\|_F^2~.
  \end{align*}
\end{lemma}
\begin{proof}
  By Taylor theorem,
  \begin{align*}
    f_{\bW}(\bx)
    &=
      f_{\btilW}(\bx)
      +
      \sum_k u_k \phi'\pr{\ip{\bx, \btilW_k}} \ip{\bx, \bW_k - \btilW_k}\\
    &+
      \underbrace{\frac12 \sum_k u_k \pr{\int_0^1 \phi''\pr{\tau \ip{\bx, \bW_k} + (1-\tau) \ip{\bx, \btilW_k}} \diff \tau} \ip{\bx, \bW_k - \btilW_k}^2}_{\epsilon(\bx)}~.
  \end{align*} 
  Cauchy-Schwarz inequality gives us
  \begin{align*}
    |\eps(\bx)| \leq \frac{B_{\phi''} \|\bx\|}{2 \sqrt{m}} \cdot \|\bW - \btilW\|_F^2~.
  \end{align*}
\end{proof}
We will use the following proposition \citep{du2018gradient,arora2019fine}:
\begin{proposition}[Concentration of \ac{NTK} gram matrix]
  \label{prop:K_concentration}
  With probability at least $1 - \delta$ over $\bW_0$,
  \begin{align*}
    \|\hat{\bK} - \bK\|_2 \leq B_{\phi'} \sqrt{\frac{\ln\pr{\frac{2 n}{\delta}}}{2 m}}~.
  \end{align*}
\end{proposition}
\begin{proof}
  Since each entry is independent, by Hoeffding's inequality we have for any $t \geq 0$,
\begin{align*}
  \P\pr{n |(\hat{\bK})_{i,j} - (\bK)_{i,j}| \geq t}
  \leq 2 e^{-2 n t^2 / B_{\phi'}^2}~,
\end{align*}
and applying the union bound
\begin{align*}
  \|\hat{\bK} - \bK\|_F^2 \leq \frac{B_{\phi'}^2 \ln\pr{\frac{2 n}{\delta}}}{2 m}~.
\end{align*}
\end{proof}
Now we are ready to prove the main Theorem of this section (in the main text we only report the second result).
\begin{nameddef}[\cref{lem:oracle_NTK} (restated)]  
  Denote
  \begin{align*}
    \bPhi_0
  \df
    \begin{bmatrix}
    u_1 \bX \, \diag\pr{\bphi'(\bX\tp \bW_{0,1})}\\
    \vdots\\
    u_m \bX \, \diag\pr{\bphi'(\bX\tp \bW_{0,m})}
  \end{bmatrix}
  \end{align*}
  and $\hat{\bK} \df \frac1n \bPhi_0\tp \bPhi_0$.
  Assume that $m \geqC (\eta T)^5$.
  Then,
  \begin{align*}
    \oracle
    =
    \sO\pr{
    \frac{1}{\eta T} \ip{\by, (n \hat{\bK})^{-1} \by}
    }
    \quad \mathrm{as} \quad \eta T \to \infty~.
  \end{align*}

  Consider \cref{ass:Activation} and that $\eta T = n$.
  Moreover, assume that entries of $\bW_0$ are i.i.d., $\bK = \E[ \hat{\bK} ~|~ S, \bu]$ with $\lmin(\bK) \geqC 1/n$,
  and assume that $\bu \sim \Udist$ independently from all sources of randomness.
  Then, with probability least $1-\delta$ over $(\initparam)$,
  \begin{align*}
    \oracle
    = \stilO_P\pr{
      \frac1n \ip{\by, (n \bK)^{-1} \by}
    }
    \quad \mathrm{as} \quad n \to \infty~.
  \end{align*}
\end{nameddef}
\begin{proof}
  The proof of the first inequality will follow by
  relaxation of the oracle \ac{R-ERM} $\oracle$ to the Moore-Penrose pseudo-inverse solution to a linearised problem given by \cref{lem:NTK}.
  The proof of the second inequality will build on the same idea, in addition making use of the concentration of entries of $\hat{\bK}$ around $\bK$.

Define
\begin{align*}
  f^{\mathrm{lin}}_{\bW}(\bx) &\df \sum_{k=1}^m u_k \phi'\pr{\ip{\bx, \bW_{0,k}}} \ip{\bW_k - \bW_{0,k}, \bx}~,\\
  \sL_S^{\mathrm{lin}}(\bW) &\df \frac12 \sum_{i=1}^n \pr{y_i - f^{\mathrm{lin}}_{\bW}(\bx)}^2~.
\end{align*}
Then for the square loss we have
\begin{align*}
  \pr{f_{\bW}(\bx_i) - y_i}^2
  &=
    \pr{f_{\bW_0}(\bx_i) + f^{\mathrm{lin}}_{\bW}(\bx_i) + \eps(\bx_i) - y_i}^2\\
  &\leq
    2 \pr{f^{\mathrm{lin}}_{\bW}(\bx_i) - (y_i - f_{\bW_0}(\bx_i))}^2 + 2 \eps(\bx_i)^2
\end{align*}
and so,
\begin{align*}
  \sL(\bW)
  \leq
  \sL^{\mathrm{lin}}(\bW) + \frac{B_{\phi''}^2}{m} \cdot \|\bW - \bW_0\|_F^4
\end{align*}
where we observe that
\begin{align*}
  \sL^{\mathrm{lin}}(\bW) = \frac{1}{n}\|\bPhi_0\tp (\bW - \bW_0) - (\by - \bhy_0)\|^2
\end{align*}
and with $\bPhi_0$, the matrix of \ac{NTK} features, defined in the statement.
Solving the above undetermined least-squares problem using the Moore-Penrose pseudo-inverse we get
\begin{align*}
  \bW^{\mathrm{pinv}} - \bW_0 = \pr{\bPhi_0 \bPhi_0\tp}^{\dagger} \bPhi_0 (\by - \bhy_0)~,
\end{align*}
and so
\begin{align*}
  \|\bW^{\mathrm{pinv}} - \bW_0\|_F^2
  &=  (\by - \bhy_0)\tp \bPhi_0\tp \pr{\bPhi_0 \bPhi_0\tp}^{\dagger2 } \bPhi_0 (\by - \bhy_0)\\
  &=  (\by - \bhy_0)\tp \pr{\bPhi_0\tp \bPhi_0}^{-1} (\by - \bhy_0)\\
  &=  (\by - \bhy_0)\tp (n \hat{\bK})^{-1} (\by - \bhy_0)
\end{align*}
where the final step can be observed by \ac{SVD} of $\Phi_0$.
Since $\sL_{S}(\bW^{\mathrm{pinv}}) = 0$,
\begin{align*}
  \oracle
  =
  \sO\pr{
  \frac{1}{\eta T} \ip{(\by - \bhy_0), (n \hat{\bK})^{-1} (\by - \bhy_0)}
  }
  \quad \mathrm{as} \quad \eta T \to \infty~.
\end{align*}
This proves the first result.

Now we prove the second result involving $\bK$.
We will first handle the empirical risk by concentration between $\hat{\bK}$ and $\bK$.
For $\balpha \in \R^n$ define $\bW_{\balpha} = \bPhi_0 \balpha + \bW_0$.
Then,
\begin{align*}
  \sL^{\mathrm{lin}}(\bW_{\balpha})
  &= \frac{1}{n}\|\bPhi_0\tp \bPhi_0 \balpha - (\by - \bhy_0)\|^2\\
  &= \frac{1}{n}\|n (\hat{\bK} - \bK) \balpha + n \bK \balpha - (\by - \bhy_0)\|^2\\
  &\leq \frac2n \|n (\hat{\bK} - \bK) \balpha\|^2 + \frac2n \|n \bK \balpha - (\by - \bhy_0)\|^2\\
  &\leq 2 n \|\hat{\bK} - \bK\|_2^2 \|\balpha\|_2^2 + \frac2n \|n \bK \balpha - (\by - \bhy_0)\|^2
\end{align*}

Plug into the above $\hat{\balpha} = (n \bK)^{-1} (\by - \bhy_0)$ (note that $\bK$ is full-rank by assumption)
\begin{align*}
  \sL^{\mathrm{lin}}(\bW_{\hat{\balpha}})
  &\leq 2 n \|\hat{\bK} - \bK\|_2^2 \|\hat{\balpha}\|_2^2\\
  &\leq n \cdot \frac{B_{\phi'}^2 \ln\pr{\frac{2 n}{\delta}}}{m} \cdot \pr{(\by - \bhy_0)\tp (n \bK)^{-2} (\by - \bhy_0)}\\
  &\leq \|\by - \bhy_0\|^2 \cdot \frac{B_{\phi'}^2 \ln\pr{\frac{2 n}{\delta}}}{m} \cdot \frac{1}{n \lmin(\bK)^2}\\
  &= 2 \sL_S(\bW_0) \cdot \frac{B_{\phi'}^2 \ln\pr{\frac{2 n}{\delta}}}{m} \cdot \frac{1}{\lmin(\bK)^2}
\end{align*}
where the last inequality hold w.p.\ at least $1-\delta$ by \cref{prop:K_concentration}.

Now we pay attention to the quadratic term within $\oracle$:
\begin{align*}
  &\|\bW_{\hat{\balpha}} - \bW_0\|_2^2
  =
  \|\bPhi_0 \hat{\balpha}\|_2^2\\
  &=
  \|\bPhi_0 (n \bK)^{-1} (\by - \bhy_0)\|_2^2\\
  &=
    (\by - \bhy_0)\tp (n \bK)^{-1} (n \hat{\bK}) (n \bK)^{-1} (\by - \bhy_0)\\
  &=
    \underbrace{(\by - \bhy_0)\tp (n \bK)^{-1} (n \hat{\bK} - n \bK) (n \bK)^{-1} (\by - \bhy_0)}_{(i)}
    +
    \underbrace{(\by - \bhy_0)\tp (n \bK)^{-1} (\by - \bhy_0)}_{(ii)}~.
\end{align*}
We will show that $(i)$ is ``small'':
\begin{align*}
  &(\by - \bhy_0)\tp (n \bK)^{-1} (n \hat{\bK} - n \bK) (n \bK)^{-1} (\by - \bhy_0)\\
  &\leq
    \|\by - \bhy_0\|^2 \|(n \bK)^{-2}\| n \|\hat{\bK} - \bK\|_2\\
  &\leq
    \|\by - \bhy_0\|^2
    \|(n \bK)^{-2}\| \cdot n
    B_{\phi'} \sqrt{\frac{\ln\pr{\frac{2 n}{\delta}}}{2 m}}\\
  &\leq
    2 \sL_S(\bW_0)
    \cdot \frac{1}{\lmin(\bK)^2} \cdot
    B_{\phi'} \sqrt{\frac{\ln\pr{\frac{2 n}{\delta}}}{2 m}}
\end{align*}
where we used \cref{prop:K_concentration} once again.
Putting all together w.p.\ at least $1-\delta$ over $\bW_0$ we have
\begin{align*}
  \oracle
  &= \sO_P\Bigg(
  \frac{1}{\eta T}\ip{(\by - \bhy_0), (n \bK)^{-1} (\by - \bhy_0)}\\
  &\quad\quad+
  \frac{2 \sL_S(\bW_0)}{\lmin(\bK)^2}  \cdot \frac{B_{\phi'}^2 \ln\pr{\frac{2 n}{\delta}}}{m}
  +  
    \frac{1}{\eta T} \cdot \frac{2 \sL_S(\bW_0)}{\lmin(\bK)^2} \cdot B_{\phi'} \sqrt{\frac{\ln\pr{\frac{2 n}{\delta}}}{2 m}}
    \Bigg) \quad \mathrm{as} \quad \eta T \to \infty~.
\end{align*}
Moreover, assuming that $\lmin(\bK) \geqC 1/n$ and $\eta T = n$, the above turns into
\begin{align*}
  \oracle
  &= \stilO_P\pr{
    \frac1n \ip{(\by - \bhy_0), (n \bK)^{-1} (\by - \bhy_0)}
    } \quad \mathrm{as} \quad n \to \infty~.
\end{align*}
The final bit is to note that
\begin{align*}
  \ip{\bhy_0, (n\bK)^{-1} \bhy_0}
  \leq
  \frac{\|\bhy_0\|_2^2}{n \lmin(\bK)}
  \leqC
  \|\bhy_0\|_2^2
\end{align*}
can be bounded w.h.p.\ by randomising $\bu \sim \Udist$:
For any $i \in [n]$ and $\delta \in (0,1)$ by Hoeffding's inequality we have:
\begin{align*}
  \P\pr{f_{\bW_0}(\bx_i) \geq B_{\phi} \sqrt{\frac{\ln\pr{\frac{1}{\delta}}}{2}}}
  &\geq
  \P\pr{f_{\bW_0}(\bx_i) \geq \sqrt{ \frac{\ln\pr{\frac{1}{\delta}}}{2} \frac1m \sum_{k=1}^m \phi\pr{\ip{(\bW_0)_k, \bx}}^2}}\\
  &\geq 1 - \delta~.
\end{align*}
Taking a union bound over $i \in [n]$ completes the proof of the second result.
\end{proof}

\pagebreak
\section{Additional Proofs}
\label{sec:additional_proofs}
\begin{nameddef}[\cref{cor:risk_noise_free} (restated)]
  Assume the same as in \cref{thm:gen_gap} and \cref{lem:OptError}.
  Then,
  \begin{align*}
    \E\br{\sL(\bW_T) \bmid \bW_0,\bu}
    \leq
    \pr{1
    +
    C \cdot \frac{\eta T}{n} \pr{ 1 + \frac{\eta T}{n} }
    }
  \E\br{\oracle
  \bmid \bW_0,\bu
  }~.
  \end{align*}
\end{nameddef}
\begin{proof}
Considering Theorem \ref{thm:gen_gap} with $t = T-1$, and noting that $\sL_{S}(\bW_T) \leq \frac{1}{T} \sum_{j=0}^{T} \sL_{S}(\bW_j)$ then yields  
\begin{align*}
  &\E[\sL(\bW_T) ~|~ \initparam]
    \leq 
    \pr{1 + b
    \pr{ \frac{\eta T }{n} + \frac{\eta^2 T^2}{n^2} }
    }
    \frac{1}{T} \sum_{j=0}^{T} \E[\sL_{S}(\bW_j) ~|~ \initparam]\\
  & \leq 
    \pr{1 \! + \! b
    \pr{ \frac{\eta T }{n} + \frac{\eta^2 T^2}{n^2} }
    }\\
  &\, \cdot \E\br{
    \min_{\bW \in \mathbb{R}^{d \times m }}
    \!
    \Big\{ 
    \sL_{S}(\bW) 
    \! + \!
    \frac{\|\bW - \bW_{0}\|_F^2}{\eta t}
    \! + \!
    \frac{\widetilde{b} }{\sqrt{m}} \cdot
    \frac{1}{T} \sum_{j=0}^{T}
    \pr{ 1 \vee \|\bW \! - \! \bW_j\|_F}^3
    \Big\}
    \bmid \initparam}
\end{align*}
where at the end we applied Lemma \ref{lem:OptError} to bound $\frac{1}{T} \sum_{j=0}^{T} \E[\sL_{S}(\bW_j)|\bW_0]$. The constants $b,\widetilde{b}$ are then defined in Theorem \ref{thm:gen_gap} and Lemma \ref{lem:OptError}. Note from smoothness of the loss we have
\[
    \|\bW - \bW_j\|_F^3 
    \leq 
    2^{3/2}\big( \|\bW - \bW_0\|_F^3  + \|\bW_0 - \bW_j\|_F^3 \big) 
    \leq 
    2^{3/2}\big( \|\bW - \bW_0\|_F^3 + (\eta j C_0)^{3/2} \big),
\]
    in particular from the properties of graident descent $\|\bW_0 - \bW_j\|_F^2 \leq \eta j \sL_{S}(\bW_0)$ for $j \in [T]$. Plugging in then yields the final bound.
  \end{proof}
  \vfill
\pagebreak

\end{document}